\documentclass{article}
\pdfoutput=1

\usepackage{graphicx} 
\usepackage[figtopcap]{subfigure}

\usepackage{array} 

\usepackage{natbib}

\usepackage{amsthm}

\usepackage{algorithm,algorithmic}

\usepackage[accepted]{icml2013}

\icmltitlerunning{Spectral Experts}

\usepackage{amsmath,amssymb}
\usepackage{soul}

\renewcommand{\comment}[1]{%
  \text{\phantom{(#1)}} \tag{#1}
}

\DeclareMathOperator {\trace} {tr}
\DeclareMathOperator{\E} {\mathbb{E}}

\DeclareMathOperator{\diag} {diag}

\DeclareMathOperator{\mult} {Multinomial}

\DeclareMathOperator{\cvec} {cvec}
\DeclareMathOperator{\vvec} {vec}

\newtheorem{condition}{Condition}

\renewcommand{\Re} {\mathbb{R}}

\newcommand{\half} {\frac{1}{2}}

\newcommand\eqdef{\ensuremath{\stackrel{\rm def}{=}}} % Equal by definition
\newcommand\refeqn[1]{(\ref{eqn:#1})}
\newcommand\sD{\ensuremath{\mathcal{D}}}

\newcommand\sE{\ensuremath{\mathcal{E}}}
\newcommand\refsec[1]{Section~\ref{sec:#1}}

\newcommand\refthm[1]{Theorem~\ref{thm:#1}}

\newcommand\sigmamin{\sigma_\text{\rm min}}

\newcommand\op{{\text{\rm op}}}

\newcommand\reflem[1]{Lemma~\ref{lem:#1}}
\newcommand{\Lop}{{\textrm{op}}}

\newcommand{\sectionref}[1] {\hyperref[#1]{Section \ref{#1}}}
\newcommand{\appendixref}[1] {\hyperref[#1]{Appendix \ref{#1}}}
\newcommand{\algorithmref}[1] {\hyperref[#1]{Algorithm \ref{#1}}}
\newcommand{\equationref}[1] {\hyperref[#1]{Equation \eqref{#1}}}
\newcommand{\figureref}[1] {\hyperref[#1]{Figure \ref{#1}}}
\newcommand{\tableref}[1] {\hyperref[#1]{Table \ref{#1}}}

\newtheorem{theorem}{Theorem}
\newtheorem{lemma}{Lemma}

\newcommand{\tp}[1] {^{\otimes #1}}
\newcommand{\innerp}[2] {\langle #1, #2 \rangle}
\DeclareMathOperator{\opX} {\mathfrak{X}}

\newcommand{\pinv}[1] {#1^{\dagger}}
\newcommand{\Ap} {\hat{A}}

\newcommand{\Wp} {\hat{W}}
\newcommand{\Winv} {W^{\dagger}}
\newcommand{\Whinv} {{\hat W}^{\dagger}}

\newcommand{\aerr}[1] {\varepsilon_{#1}}

\newcommand{\serr}[1] {\alpha_{#1}}

\usepackage{etoolbox}
\newtoggle{withappendix}
\toggletrue{withappendix}
\begin{document} 

\twocolumn[
%\icmltitle{Spectral Experts: A spectral algorithm for mixtures of linear regressions}
\icmltitle{Spectral Experts for Estimating Mixtures of Linear Regressions}
%\icmltitle{Efficient Consistent Estimation for Mixtures of Linear Regression}
%\icmltitle{Two Tensors Suffice: Efficient Consistent Estimation for Mixtures of Linear Regression}

% unless you've provided the [accepted] option to the icml2013
\icmlauthor{Arun Tejasvi Chaganty}{chaganty@cs.stanford.edu}
\icmlauthor{Percy Liang}{pliang@cs.stanford.edu}
\icmladdress{Stanford University, Stanford, CA 94305 USA}

\icmlkeywords{spectral algorithms, mixture of experts, latent-variable
models, machine learning, ICML}

\vskip 0.3in
]

\begin{abstract}

Discriminative latent-variable models are typically learned using
EM or gradient-based optimization, which suffer from local optima.
In this paper, we develop a new computationally efficient
and provably consistent estimator for a mixture of linear regressions,
a simple instance of a discriminative latent-variable model.
Our approach relies on a low-rank linear regression to recover
a symmetric tensor, which can be factorized into the parameters
using a tensor power method.
We prove rates of convergence for our estimator
and provide an empirical evaluation illustrating
its strengths relative to local optimization (EM).

%  Spectral algorithms for latent variable models have seen considerable
%  recent interest for being efficient consistent estimators of model
%  parameters. These algorithms make few if any assumptions about the
%  generative process of the data, while providing a polynomial sample
%  and computational complexity. We present a new spectral algorithm for
%  a discriminative model, a mixture of linear regressions and show that
%  it can recover the regression coefficients with similar polynomial
%  guarantees. We evaluate the algorithm on linearly and non-linearly
%  generated data, as well as on a motion tracking task and compare it's
%  characteristics with an E-M algorithm.
\textbf{Last Modified: \today}
\end{abstract} 

\section{Introduction}
\label{sec:intro}

% Discriminative models work better in practice Discriminative models
% with latent variables add expressive power (examples)
Discriminative latent-variable models,
which combine the high accuracy of discriminative models
with the compact expressiveness of latent-variable models,
have been widely applied to many tasks, including
object recognition \cite{quattoni04crf},
human action recognition \cite{wang09crf},
syntactic parsing \cite{petrov08discriminative},
and machine translation \cite{liang06discrimative}.
However, parameter estimation in these models is difficult;
past approaches rely on local optimization (EM, 
gradient descent) and are vulnerable to local optima.

% Goal of this paper
% This paper: first step in establishing provably correct estimation
% with simple mixture of linear regression
Our broad goal is to develop efficient provably consistent estimators for
discriminative latent-variable models.
In this paper, we provide a first step in this 
direction by proposing a new algorithm for a simple model,
\emph{a mixture of linear regressions} \cite{VieleTong2002}. % (MLG) ?

% Recent method of moments spectral methods sidestep local optima
Recently, method of moments estimators have been developed for
\emph{generative} latent-variable models, including
mixture models, HMMs \cite{anandkumar12moments},
Latent Dirichlet Allocation \cite{anandkumar12lda},
and parsing models \cite{hsu12identifiability}.
The basic idea of these methods is to express
the unknown model parameters as a tensor factorization
of the third-order moments of the model distribution, a quantity
which can be estimated from data.
The moments have a special symmetric structure
which permits the factorization to be computed efficiently using the robust
tensor power method \cite{AnandkumarGeHsu2012}.

% Twist
In a mixture of linear regressions, using third-order moments does not
directly reveal the tensor structure of the problem, so we cannot
simply apply the above tensor factorization techniques.  Our approach
is to employ low-rank linear regression
\cite{NegahbanWainwright2009,Tomioka2011} to predict the second and third powers of
the response.  The solution to these regression problems provide the appropriate symmetric tensors,
on which we can then apply the tensor power method to retrieve the final parameters.

The result is a simple and efficient two-stage algorithm,
which we call Spectral Experts.
We prove that our algorithm yields consistent parameter estimates under certain
identifiability conditions.  We also conduct an empirical evaluation
of our technique to understand its statistical properties (\sectionref{sec:evaluation}).
%In particular, we find that in the low data regime,
%but in the high data regime,
%Spectral Experts outperforms EM.
While Spectral Experts generally does not outperform EM, presumably due to its
weaker statistical efficiency, it serves as an effective initialization for EM,
significantly outperforming EM with random initialization.

%We evaluate our algorithm on simulated linear and non-linear data to
%understand how well the algorithm scales with the number of components,
%dimensions and separation.

%To show how robust the algorithm is to model
%misspecification, we evaluate it's performance on \todo{a motion
%tracking dataset(?)}. 

% What properties of the problem make it hard on EM versus spectral?
%  Separation, dimensionality, number of clusters, etc.
% Real dataset?

%We compare our approach with expectation-maximization; our results,
%presented in \sectionref{sec:evaluation} shows that spectral methods do
%indeed provide an answer close enough to the global optima that using
%a local method from there recovers the true parameters. \todo{Describe
%some of our other high-level findings.} 

%PSL: not for this paper
%\todo{What about talking about initialization? Generalizing and kmeans++} 

\subsection{Notation}

%We have already defined $y$ to be the response variable and $x \in
%\Re^d$ to be the covariates it depends on. Note that the matrix of
%regressors $B$ is a $d \times k$ matrix, where the $k$-th column is
%$\beta_k$.

Let $[n] = \{ 1, \dots, n \}$ denote the first $n$ positive integers.
We use $O(f(n))$ to denote a function $g(n)$ such that $\lim_{n \to\infty} g(n)/f(n) < \infty$.
%We will only hide universal constants, exposing all dependence on the number of samples,
%dimensionality, norms, etc.

We use $x\tp{p}$ to represent the $p$-th order tensor formed by taking
the tensor product of $x \in \Re^d$; i.e. $x\tp{p}_{i_1 \ldots i_p}
= x_{i_1} \cdots x_{i_p}$. We will use $\innerp{\cdot}{\cdot}$ to denote
the generalized dot product between two $p$-th order tensors:
$\innerp{X}{Y} = \sum_{i_1, \ldots i_p} X_{i_1, \ldots i_p} Y_{i_1,
\ldots i_p}$.  A tensor $X$ is symmetric if for all $i,j \in [d]^p$
which are permutations of each other, $X_{i_1 \cdots i_p}$ = $X_{j_1
\cdots j_p}$ (all tensors in this paper will be symmetric).  For
a $p$-th order tensor $X \in (\Re^d)\tp{p}$, the mode-$i$ unfolding of
$X$ is a matrix $X_{(i)} \in \Re^{d \times d^{p-1}}$, whose $j$-th row contains all the elements of $X$ whose
$i$-th index is equal to $j$. 
%For $X \in (\Re^d)\tp{p}$ and $Y \in (\Re^d)\tp{(p+q)}$,% $p, q > 0$,
%$\innerp{X}{Y} \in \Re^{d^q}$ is the projection of $X$ onto the first
%$p$ modes of $Y$, i.e. $\innerp{X}{Y}_{i_1, \ldots, i_q} = \sum_{j_1,
%\ldots, j_p} X_{j_1, \ldots, j_p} Y_{j_1, \ldots, j_p, i_1, \ldots,
%i_q}$.

% Norms
For a vector $X$,
let $\|X\|_\op$ denote the 2-norm.
For a matrix $X$,
let $\|X\|_*$ denote the nuclear (trace) norm (sum of singular values),
$\|X\|_F$ denote the Frobenius norm (square root of sum of squares of singular values),
$\|X\|_{\max}$ denote the max norm (elementwise maximum),
$\|X\|_\op$ denote the operator norm (largest singular value), and
$\sigma_k(X)$ be the $k$-th largest singular value of $X$.
For a $p$-th order tensor $X$,
let $\|X\|_* = \frac{1}{p} \sum_{i=1}^p \|X_{(i)}\|_*$ denote
the average nuclear norm over all $p$ unfoldings,
and let $\|X\|_\op = \frac{1}{p} \sum_{i=1}^p \|X_{(i)}\|_\op$
denote the average operator norm over all $p$ unfoldings.

Let $\vvec(X)$ be the vectorization of a $p$-th order tensor. For
example, if $X \in (\Re^{2})\tp{3}$, $\vvec(X) = (X_{111}, X_{112},
\cdots, X_{222})$.
For a tensor $X \in (\Re^d)\tp{p}$, let $\cvec(X) \in
\Re^{N(d,p)}, N(d,p) = \binom{d + p - 1}{p}$ be the collapsed vectorization of
$X$. For example, if $X \in \Re^{d \times d}$, $\cvec(X)
= (X_{ii} : i \in [d]; \frac{X_{ij} + X_{ji}}{\sqrt{2}} : i,j \in [d], i<j)$.
In general, each component of $\cvec(X)$ is indexed by a vector of
counts $(c_1, \dots, c_d)$ with total sum $\sum_i c_i = p$.  The value
of that component is $\frac{1}{\sqrt{|K(c)|}} \sum_{k \in K(c)} X_{k_1 \cdots k_p}$, where $K(c)
= \{ k \in [d]^p : \forall i \in [d], c_i = |\{ j \in [p] : k_j = i \}|
\}$ are the set of index vectors $k$ whose count profile is $c$.
We note that for a symmetric tensor $X$ and any tensor $Y$,
$\innerp{X}{Y} = \innerp{\cvec(X)}{\cvec(Y)}$; this property is not true
in general though.
Later, we'll see that vectorization allow us to perform regression on tensors,
and collapsing simplifies our identifiability condition.

\section{Model}
\label{sec:model}

\newcommand{\xn}[1]{x^{(#1)}}
\newcommand{\xni}{\xn{i}}
\newcommand{\yn}[1]{y^{(#1)}}
\newcommand{\yni}{\yn{i}}

The mixture of linear regressions model \citep{VieleTong2002} defines
a conditional distribution over a response $y \in \Re$
given covariates $x \in \Re^d$.
Let $k$ be the number of mixture components.
The generation of $y$ given $x$ involves three steps:
(i) draw a mixture component $h \in [k]$ according to mixture proportions
$\pi = (\pi_1, \dots, \pi_k)$;
(ii) draw observation noise $\epsilon$ from a known zero-mean noise distribution $\sE$,
and (iii) set $y$ deterministically based on $h$ and $\epsilon$.
More compactly: %to $\beta_h^\top x$ plus some observation noise $\epsilon$.
\begin{eqnarray}
  h &\sim& \mult(\pi), \\
  \epsilon &\sim& \sE, \\
  y &=& \beta_{h}^T x + \epsilon.
\end{eqnarray}
The parameters of the model are $\theta = (\pi, B)$,
where $\pi \in \Re^d$ are the mixture proportions and
$B = [\beta_1 \mid \dots \mid \beta_k] \in \Re^{d \times k}$
are the regression coefficients.
Note that the choice of mixture component $h$ and the observation noise $\epsilon$ are independent.
%We also assume the distribution of the observation noise is known and has bounded support.
The learning problem is stated as follows:
given $n$ i.i.d.\ samples $(\xn{1}, \yn{1}), \dots, (\xn{n}, \yn{n})$
drawn from the model with some unknown parameters $\theta^*$,
return an estimate of the parameters $\hat\theta$.

% History of model
The mixture of linear regressions model has been applied
in the statistics literature for modelling music perception, where $x$ is the
actual tone and $y$ is the tone perceived by a musician \cite{VieleTong2002}.
% Reference mixture of experts.
% Acknowledge that mixture proportions % can't depend on $x$
The model is an instance of the hierarchical mixture of experts
\cite{jacobs91experts}, in which the mixture proportions are allowed to depend
on $x$, known as a gating function.
This dependence allow the experts to be localized in input space,
providing more flexibility, but we do not consider this dependence in our model.

% Typically people use EM, but the maximum marginal likelihood is
% non-convex.
The estimation problem for a mixture of linear regressions is difficult because
the mixture components $h$ are unobserved,
resulting in a non-convex log marginal likelihood.
The parameters are typically learned using
expectation maximization (EM) or Gibbs sampling \cite{VieleTong2002},
which suffers from local optima.
In the next section, we present a new algorithm
that sidesteps the local optima problem entirely.
%(we show this empirically in \sectionref{sec:evaluation}).
%We provide several instances where this leads
%to poor local optima in \sectionref{sec:evaluation}.

\section{Spectral Experts algorithm}
\label{sec:algo}

In this section, we describe our Spectral Experts algorithm
for estimating model parameters $\theta = (\pi, B)$.
The algorithm consists of two steps:
(i) low-rank regression to estimate certain symmetric tensors;
and (ii) tensor factorization to recover the parameters.
The two steps can be performed efficiently using
convex optimization and tensor power method, respectively.

%%% first moment
To warm up, let us consider linear regression
on the response $y$ given $x$.
From the model definition, we have $y = \beta_h^\top x + \epsilon$.
The challenge is that the regression coefficients $\beta_h$ depend on the random $h$.
%$\epsilon \sim \normal{0}{\E[\epsilon^2]}$, and $h$ is a random quantity
%independent of $x$.
The first key step is to average over this randomness by defining
average regression coefficients
$M_1 \eqdef \sum_{h=1}^k \pi_h \beta_h$.
Now we can express $y$ as a linear function of $x$ with non-random coefficients $M_1$
plus a noise term $\eta_1(x)$:
\begin{align}
  y &= \innerp{M_1}{x} +
  \underbrace{(\innerp{\beta_h - M_1}{x} + \epsilon)}_{\eqdef \eta_1(x)}. \label{eqn:y1}
\end{align}
The noise $\eta_1(x)$ is the sum of two terms:
(i) the \emph{mixing noise} $\innerp{M_1 - \beta_h}{x}$
due to the random choice of the mixture component $h$,
and (ii) the \emph{observation noise} $\epsilon \sim \sE$.
Although the noise depends on $x$,
it still has zero mean conditioned on $x$.
We will later show that we can
perform linear regression on the data $\{\xni,
\yni\}_{i=1}^{n}$ to produce a consistent estimate of $M_1$.
%under identifiability conditions.
But clearly, knowing $M_1$ is insufficient
for identifying all the parameters $\theta$,
as
$M_1$ only contains $d$ degrees of freedom whereas $\theta$ contains $O(kd)$.

%%% second moments
Intuitively, performing regression on $y$ given $x$ provides only first-order
information.  The second key insight is that we can perform regression
on higher-order powers to obtain more information about the parameters.
Specifically, for an integer $p \ge 1$, let us define the average
$p$-th order tensor power of the parameters as follows:
\begin{align}
M_p &\eqdef \sum_{h=1}^k \pi_h \beta_h\tp{p}. \label{eqn:Mp} % \in (\Re^{d})\tp{p}.
\end{align}
Now consider performing regression on $y^2$ given $x\tp{2}$.
Expanding $y^2 = (\innerp{\beta_h}{x} + \epsilon)^2$,
using the fact that $\innerp{\beta_h}{x}^p = \innerp{\beta_h\tp{p}}{x\tp{p}}$,
we have:
\begin{align}
  y^2 &= \innerp{M_2}{x\tp{2}} + \E[\epsilon^2] + \eta_2(x), \label{eqn:y2} \\
\eta_2(x) &= \innerp{\beta_h\tp{2} - M_2}{x\tp{2}} + 2 \epsilon \innerp{\beta_h}{x} + (\epsilon^2 - \E[\epsilon^2]). \nonumber
\end{align}
Again, we have expressed $y^2$ has a linear function of $x\tp{2}$
with regression coefficients $M_2$, plus a known bias $\E[\epsilon^2]$ and noise.\footnote{If $\E[\epsilon^2]$ were not known,
we could treat it as another coefficient
to be estimated.  The coefficients $M_2$ and $\E[\epsilon^2]$ can be estimated jointly
provided that $x$ does not already contain a bias ($x_j$ must be non-constant for every $j \in [d]$).}
Importantly, the noise has mean zero; 
in fact each of the three terms has zero mean
by definition of $M_2$ and independence of $\epsilon$ and $h$.

Performing regression yields a consistent estimate of $M_2$,
but still does not identify all the parameters $\theta$.
In particular, $B$ is only identified up to rotation:
if $B = [\beta_1 \mid \cdots \mid \beta_k]$ satisfies
$B \diag(\pi) B^\top = M_2$ and $\pi$ is uniform, then $(B Q) \diag(\pi) (Q^\top B^\top) = M_2$
for any orthogonal matrix $Q$.

%%% third moment
Let us now look to the third moment for additional information.
We can write $y^3$ as a linear function of $x\tp{3}$ with coefficients $M_3$,
a known bias $3 \E[\epsilon^2] \innerp{\hat M_1}{x} + \E[\epsilon^3]$ and some noise $\eta_3(x)$:
\begin{align}
  y^3 &= \innerp{M_3}{x\tp{3}} + 3\E[\epsilon^2] \innerp{\hat M_1}{x} + \E[\epsilon^3] + \eta_3(x), \nonumber \\
\eta_3(x) &= \innerp{\beta_h\tp{3} - M_3}{x\tp{3}}
+ 3 \epsilon \innerp{\beta_h\tp{2}}{x\tp{2}} \label{eqn:y3} \\
&\quad + 3(\epsilon^2 \innerp{\beta_h}{x} - \E[\epsilon^2] \innerp{\hat M_1}{x})
+ (\epsilon^3 - \E[\epsilon^3]). \nonumber
\end{align}
The only wrinkle here is that $\eta_3(x)$ does not quite have zero mean.
It would if $\hat M_1$ were replaced with $M_1$, but $M_1$ is not available to us.
Nonetheless, as $\hat M_1$ concentrates around $M_1$, the noise bias will go to zero.
Performing this regression yields an estimate of $M_3$.
We will see shortly that knowledge of $M_2$ and $M_3$ are sufficient to recover
all the parameters.

\begin{algorithm}[t]
  \caption{Spectral Experts}
  \label{algo:spectral-experts}
  \begin{algorithmic}[1]
    \INPUT Datasets $\mathcal{D}_p = \{ (\xn{1}, \yn{1}), \cdots, (\xn{n}, \yn{n}) \}$ for $p = 1, 2, 3$;
    %regularization strengths $\lambda_n^{(2)} = \frac{c_2}{\sqrt{n}}$, $\lambda_n^{(3)} = \frac{c_3}{\sqrt{n}}$;
    regularization strengths $\lambda_n^{(2)}$, $\lambda_n^{(3)}$;
    observation noise moments $\E[\epsilon^2], \E[\epsilon^3]$.
    \OUTPUT Parameters $\hat\theta = (\hat \pi, [\hat \beta_1 \mid \cdots \mid \hat \beta_k])$.
    \STATE Estimate compound parameters $M_2, M_3$ using \textbf{low-rank regression}:
    \begin{align}
      &\hat M_1 = \arg\min_{M_1} \label{eqn:estimateM1} \\
      &\quad\frac{1}{2n}\sum_{(x,y) \in \sD_1} (\innerp{M_1}{x} - y)^2, \nonumber \\
      &\hat M_2 = \arg\min_{M_2} \quad \lambda_n^{(2)} \|M_2\|_* + \label{eqn:estimateM2} \\
      &\quad\frac{1}{2n}\sum_{(x,y) \in \sD_2} (\innerp{M_2}{x\tp{2}} + \E[\epsilon^2] - y^2)^2, \nonumber \\
      &\hat M_3 = \arg\min_{M_3} \quad \lambda_n^{(3)} \|M_3\|_* + \label{eqn:estimateM3} \\
      % NOTE: hspace added to make ICML accept our paper.
      &{\frac{1}{2n} \hspace{-0.5em}\sum_{(x,y) \in \sD_3} \hspace{-1em}(\innerp{M_3}{x\tp{3}} + 3 \E[\epsilon^2]\innerp{\hat M_1}{x} + \E[\epsilon^3] - y^3)^2}. \nonumber
    \end{align}
    \STATE Estimate parameters $\theta = (\pi, B)$ using \textbf{tensor factorization}:
    \begin{enumerate}
      \item [(a)] Compute whitening matrix $\hat W \in \Re^{d \times k}$ (such that $\hat W^\top
      \hat M_2 \hat W = I$) using SVD.
      \item [(b)] Compute eigenvalues $\{\hat a_h\}_{h=1}^k$
      and eigenvectors $\{\hat v_h\}_{h=1}^k$
      of the whitened tensor $\hat M_3(\hat W, \hat W, \hat W) \in \Re^{k \times k \times k}$
      by using the robust tensor power method.
    \item [(c)] Return parameter estimates $\hat\pi_h = \hat a_h^{-2}$
    and $\hat\beta_h = (\hat W^{\top})^\dagger (\hat a_h \hat v_h)$.
    \end{enumerate}
  \end{algorithmic}
\end{algorithm}

% Full algorithm
Now we are ready to state our full algorithm, which we call Spectral Experts
(\algorithmref{algo:spectral-experts}).
First, we perform three regressions to recover the \emph{compound parameters}
$M_1$ \refeqn{y1},
$M_2$ \refeqn{y2}, and
$M_3$ \refeqn{y3}.
Since $M_2$ and $M_3$ both only have rank $k$,
we can use nuclear norm regularization
\cite{Tomioka2011,NegahbanWainwright2009}
to exploit this low-rank structure and improve our compound parameter estimates.
In the algorithm, the regularization strengths $\lambda_n^{(2)}$ and $\lambda_n^{(3)}$
are set to $\frac{c}{\sqrt{n}}$ for some constant $c$.
%The resulting semidefinite program is a standard one which has received
%much attention in recent years.
% We use a standard off-the-shelf SDP solver, CVX, to solve .
% We use a rather simple proximal gradient-based approach,
% in which the nuclear norm is handled in closed form by taking an SVD
% and soft-thresholding the singular values \cite{donoho95soft,cai10soft}.

% Tensor factorization
Having estimated the compound parameters $M_1$, $M_2$ and $M_3$, it
remains to recover the original parameters $\theta$.
\citet{AnandkumarGeHsu2012} showed that for $M_2$ and $M_3$ of
the forms in \refeqn{Mp}, it is possible to efficiently accomplish this.
Specifically, we first compute a whitening matrix $W$ based on the SVD of $M_2$
and use that to construct a tensor $T = M_3(W, W, W)$ whose factors are orthogonal.
We can use the robust tensor power method to compute all the
eigenvalues and eigenvectors of $T$, from which it is easy to recover
the parameters $\pi$ and $\{\beta_h\}$.

\paragraph{Related work}

% Spectral
In recent years, there has a been a surge of interest in ``spectral'' methods
for learning latent-variable models.  One line of work has
focused on observable operator models \cite{hsu09spectral,song10kernel,parikh12spectral,cohen12pcfg,balle11transducer,balle12automata}
in which a re-parametrization of the true parameters are recovered,
which suffices for prediction and density estimation.
Another line of work is based on the method of moments and uses eigendecomposition of a certain tensor
to recover the parameters \cite{anandkumar12moments,anandkumar12lda,hsu12identifiability,hsu13spherical}.
Our work extends this second line of work to models that
require regression to obtain the desired tensor.

% Unmixing
In spirit, Spectral Experts bears some resemblance to the unmixing
algorithm for estimation of restricted PCFGs
\cite{hsu12identifiability}.
In that work, the observations (moments) provided a linear combination over
the compound parameters.  ``Unmixing'' involves solving for the compound
parameters by inverting a mixing matrix.
In this work,
each data point (appropriately transformed) provides a different noisy projection of
the compound parameters.

% 
%Previous work has focused on using spectral methods to learn finite state automata and transducers.
Other work has focused on learning discriminative models,
notably \citet{balle11transducer} for finite state transducers (functions from strings to strings),
and \citet{balle12automata} for weighted finite state automata (functions from strings to real numbers).
Similar to Spectral Experts,
\citet{balle12automata} used a two-step approach,
where convex optimization is first used to estimate moments (the Hankel matrix in their case),
after which these moments are subjected to spectral decomposition.  
However, these methods are developed in the observable operator framework, whereas we consider parameter estimation.

% Signal
The idea of performing low-rank regression on $y^2$ has been explored
in the context of signal recovery from magnitude measurements
\cite{candes11phaselift,ohlsson12phase}.
There, the actual observed response was $y^2$,
whereas in our case, we deliberately construct powers $y,y^2,y^3$
to identify the underlying parameters.

%\citet{AnandkumarGeHsu2012} describes an approach that uses
%rotates $M_3$ to an orthogonal basis by using the whitening transform of
%$M_2$. The eigenvectors and eigenvalues recovered from the
%eigendecomposition of $M_3(W, W, W)$ can be de-whitened to recover the
%$\beta_k$ and $\pi_k$.

% Describe the rest of the algorithm.
%This description completes a sketch of the algorithm, described in
%\algorithmref{algo:spectral-experts}. Going ahead, we have yet to show
%that the regression is well-behaved which we will do in
%\sectionref{sec:regression}. This is of concern because the regression
%problem we have has variance introduced from component selection,
%independent of any observation noise. We will show that we can indeed
%efficiently recover $M_2$ and $M_3$ using ideas from low-rank
%regression. Finally, we will outline the tensor power method to recover
%$B$ and $\pi$ given these two quantities, $M_2$ and $M_3$ in
%\sectionref{sec:tensor-power}. 

%%%%%%%%%%%%%%%%%%%%%%%%%%%%%%%%%%%%%%%%%%%%%%%%%%%%%%%%%%%%
\section{Theoretical results}
\label{sec:theory}

In this section, we provide theoretical guarantees for the Spectral Experts algorithm.
Our main result shows that the parameter estimates $\hat\theta$ converge to $\theta$
at a $\frac{1}{\sqrt{n}}$ rate that depends polynomially on the bounds on the
parameters, covariates, and noise, as well the $k$-th smallest singular values
of the compound parameters and various covariance matrices.

\begin{theorem}[Convergence of Spectral Experts]
\label{thm:convergence}
Assume each dataset $\sD_p$ (for $p = 1, 2, 3$) consists of $n$ i.i.d.\ points independently drawn from a mixture
of linear regressions model with parameter $\theta^*$.\footnote{Having three independent copies simplifies the analysis.}
Further, assume 
$\|x\|_2 \le R$, 
$\|\beta_h^*\|_2 \le L$ for all $h \in [k]$,
$|\epsilon| \le S$
and $B$ is rank $k$.
Let $\Sigma_p \eqdef \E[\cvec(x\tp{p})\tp{2}]$, 
and assume $\Sigma_p \succ 0$ for each $p \in \{1,2,3\}$.
Let $\epsilon < \half$.
Suppose the number of samples is
$n = \max(n_1,n_2)$
where 
\begin{align*}
n_1 &= \Omega \left(\frac{R^{12} \log(1/\delta)}{\min_{p \in [3]} \sigmamin(\Sigma_p)^2} \right) \\
n_2 &= \Omega \left(\epsilon^{-2}~ \frac{k^2 \pi^2_{\max} \|M_2\|_\op^{1/2} \|M_3\|_\op^2 { L^{6} S^{6} R^{12}}}{\sigma_k(M_2)^{5} {\sigmamin(\Sigma_1)^2}} \log(1/\delta) \right).
\end{align*}
If each regularization strength $\lambda_n^{(p)}$ is set to 
$$\Theta\left( \frac{L^p S^p R^{2p}}{\sigmamin(\Sigma_1)^2} \sqrt{\frac{\log(1/\delta)}{n}} \right),$$
for $p \in 2, 3$,
%$O\left(\frac{\sigmamin(\Sigma_p)}{\sqrt{k}}~ \epsilon \right)$, 
%$\Omega\left(\sigma^3 L^3 R^6 \sqrt{\frac{\log(1/\delta)}{n}}\right)$,
then the parameter estimates $\hat\theta = (\hat\pi, \hat B)$ returned by
\algorithmref{algo:spectral-experts} (with the columns appropriately permuted)
satisfies 
  \begin{align*}
  \|\hat \pi - \pi \|_{\infty} \le \epsilon \quad\quad 
  %&= O\left(\frac{k \pi_{\max}^{5/2}\| {M_3} \|_\op}{\sigma_k(M_2)^{5/2}} ~ \epsilon \right) \\
  \|\hat \beta_h - \beta_h\|_2 \le \epsilon
  %&= O\left( \frac{k \pi_{\max} \|M_2\|_\op^{1/2} \| {M_3} \|_\op}{\sigma_k(M_2)^{5/2}}~ \epsilon \right),
  \end{align*}
  for all $h \in [k]$.
%$\|\hat\pi - \pi^*\|_{\infty} \le \epsilon$
%and for all $h \in [k]$,
%$\|\hat\beta_h - \beta^*_h\|_2 \le \frac{\epsilon}{\sqrt{\pi_h^*}}$.
\end{theorem}

While the dependence on some of the norms ($L^6,S^6,R^{12}$) looks formidable,
it is in some sense unavoidable, since we
need to perform regression on third-order moments.
Classically, the number of samples required is squared norm of the covariance matrix,
which itself is bounded by the squared norm of the data, $R^3$. This
third-order dependence also shows up in the regularization strengths;
the cubic terms bound each of $\epsilon^3$,
$\beta_h^3$ and $\|(x\tp{3})\tp{2}\|_F$ with high probability. 

The proof of the theorem has two parts.
First, we bound the error in the compound parameters estimates $\hat M_2,\hat M_3$
using results from \citet{Tomioka2011}.
Then we use results from \citet{AnandkumarGeHsu2012} to convert this error
into a bound on the actual parameter estimates $\hat\theta = (\hat\pi, \hat B)$
derived from the robust tensor power method.
But first, let us study a more basic property: identifiability.

%%%%%%%%%%%%%%%%%%%%%%%%%%%%%%%%%%%%%%%%%%%%%%%%%%%%%%%%%%%%

\subsection{Identifiability from moments}

In ordinary linear regression, the regression coefficients $\beta \in
\Re^d$ are identifiable if and only if the data has full rank:
$\E[x\tp{2}] \succ 0$, and furthermore, identifying $\beta$ requires
only moments $\E[xy]$ and $\E[x\tp{2}]$ (by observing the optimality
conditions for \refeqn{y1}).  However, in mixture of linear regressions,
these two moments only allow us to recover $M_1$.  \refthm{convergence}
shows that if we have the higher order analogues, $\E[x\tp{p}y\tp{p}]$
and $\E[x\tp{2p}]$ for $p \in \{1,2,3\}$, we can then identify the
parameters $\theta = (\pi, B)$,
provided the following \emph{identifiability condition} holds: $\E[\cvec(x\tp{p})\tp{2}] \succ
0$ for $p \in \{1,2,3\}$.

This identifiability condition warrants a little care,
as we can run into trouble when components of $x$ are dependent on each other
in a particular algebraic way.
For example, suppose $x = (1, t, t^2)$, the common polynomial
basis expansion, so that all the coordinates are deterministically
related.  While $\E[x\tp{2}] \succ 0$ might be satisfied (sufficient for ordinary linear regression),
$\E[\cvec(x\tp{2})\tp{2}]$ is singular for
any data distribution.
To see this, note that $\cvec(x\tp{2}) = [1 \cdot 1, t\cdot t, 2(1
\cdot t^2), 2(t \cdot t^2), (t^2 \cdot t^2)]$ contains components $t
\cdot t$ and $2(1 \cdot t^2)$, which are linearly dependent.  Therefore,
Spectral Experts would not be able to identify the parameters of
a mixture of linear regressions for this data distribution.

We can show that some amount of unidentifiability is intrinsic to
estimation from low-order moments, not just an artefact of our
estimation procedure.  Suppose $x = (t, \dots, t^d)$.  Even if we
observed all moments $\E[x\tp{p}y\tp{p}]$ and $\E[x\tp{2p}]$ for $p \in
[r]$ for some $r$, all the resulting coordinates would be monomials of $t$ up to only degree
$2dr$, and thus the moments live in a $2dr$-dimensional subspace.  On
the other hand, the parameters $\theta$ live in a subspace of at least
dimension $dk$.  Therefore, at least $r \ge k/2$ moments are required
for identifiability of any algorithm for this monomial example.

%%%%%%%%%%%%%%%%%%%%%%%%%%%%%%%%%%%%%%%%%%%%%%%%%%%%%%%%%%%%
\subsection{Analysis of low-rank regression}
\label{sec:regression}

In this section, we will bound the error of
the compound parameter estimates $\|\Delta_2\|_F^2$ and $\|\Delta_3\|_F^2$,
where $\Delta_2 \eqdef \hat M_2 - M_2$
and $\Delta_3 \eqdef \hat M_3 - M_3$.
Our analysis is based on the low-rank regression framework of
\citet{Tomioka2011} for tensors, which builds on
\citet{NegahbanWainwright2009} for matrices.
The main calculation involved is controlling the noise $\eta_p(x)$,
which involves various polynomial combinations of the mixing noise and observation noise.

Let us first establish some notation that unifies the three regressions (\refeqn{estimateM1}, \refeqn{estimateM2}, and \refeqn{estimateM3}).
Define the observation operator $\opX_p(M_p) : \Re^{d\tp{p}} \to \Re^{n}$
mapping compound parameters $M_p$:
\begin{align}
\opX_p(M_p; \sD)_i &\eqdef \innerp{M_p}{x\tp{p}_i}, & (x_i, y_i) \in \sD.
\end{align}

Let $\kappa(\opX_p)$ be the restricted strong convexity constant,
and let $\opX^*_p(\eta_p; \sD) = \sum_{(x,y) \in \sD} \eta_p(x) x\tp{p}$
be the adjoint.

\begin{lemma}[\citet{Tomioka2011}, Theorem 1]
\label{lem:lowRank}
Suppose there exists a restricted strong convexity constant $\kappa(\opX_p)$ such that
$$\frac{1}{n} \| \opX_p( \Delta )\|_2^2 \ge \kappa(\opX_p) \|\Delta\|^2_F \quad \text{and} \quad
\lambda^{(p)}_n \ge \frac{2 \|\opX_p^*(\eta_p)\|_\op}{n}.$$
Then the error of $\hat M_p$ is bounded as follows:
$$\| \hat M_p - M_p \|_F \le \frac{32 \lambda^{(p)}_n \sqrt{k}}{\kappa(\opX_p)}.$$
\end{lemma}

Going forward, we need to lower bound the restricted strong convexity
constant $\kappa(\opX_p)$ and upper bound the operator norm of the adjoint operator
$\|\opX_p^*(\eta_p)\|_\op$. The proofs of the following lemmas follow
from standard concentration inequalities and are detailed in 
\iftoggle{withappendix}{
\appendixref{sec:proofs:regression}.
}{
the supplementary material.
}
%have been deferred to \appendixref{sec:proofs}. 

%We will appeal to the random design framework that models the input $x$
%as random and show bounds that hold with high probability.
%\paragraph{Adjoint operator}
%In this section, we upper bound the operator norm of the adjoint
%$\|\opX_p(\eta_p)\|_\op$.

% First, let us lower bound the restricted strong convexity parameter $\kappa(\opX_p)$:

\begin{lemma}[lower bound on restricted strong convexity constant]
\label{lem:lowRankLower}
%Let $\Sigma_p \eqdef \E[\cvec(x\tp{p})\tp{2}]$.
If $$n = \Omega \left(\max_{p\in [3]} \frac{R^{4p} (p!)^2 \log(1/\delta)}{\sigmamin(\Sigma_p)^2} \right),$$
then with probability at least $1-\delta$:
$$\kappa(\opX_p) \ge \frac{\sigmamin(\Sigma_p)}{2},$$
for each $p \in [3]$.
\end{lemma}

% Next, let us upper bound the operator norm of the observation adjoint:

\begin{lemma}[upper bound on adjoint operator]
\label{lem:lowRankUpper}
% Let $\opX_p$ be the linear operator previously defined. 
If $$n = \Omega \left(\max_{p\in [3]} \frac{L^{2p} S^{2p} R^{4p} \log(1/\delta)}{ \sigmamin(\Sigma_1)^2 \left(\lambda_n^{(p)}\right)^2} \right),$$
then with probability at least $1-\delta$:
$$\lambda_n^{(p)} \ge \frac1{n} \|\opX_p^*(\eta_p)\|_\op,$$
for each $p \in [3]$.
\end{lemma}

%%%%%%%%%%%%%%%%%%%%%%%%%%%%%%%%%%%%%%%%%%%%%%%%%%%%%%%%%%%%
\subsection{Analysis of the tensor factorization} 
\label{sec:tensorError}

Having bounded the error of the compound parameter estimates $\hat M_2$ and $\hat M_3$,
we will now study how this error propagates through the tensor factorization step of 
\algorithmref{algo:spectral-experts},
which includes whitening, applying the robust tensor power method \cite{AnandkumarGeHsu2012},
and unwhitening.
\begin{lemma}
  \label{lem:tensorPower}
  Let $M_3 = \sum_{h=1}^{k} \pi_h \beta_h\tp{3}$.
  Let $\|\hat M_2 - M_2\|_\op$ and $\|\hat M_3 - M_3\|_\op$ both be less than
  \vspace{-0.5em}
  $$\frac{\sigma_k(M_2)^{5/2}}{k \pi_{\max} \|M_2\|_\op^{1/2} \| {M_3} \|_\op}~ \epsilon,$$
  for some $\epsilon < \half$. 
  Then, there exists a permutation of indices such that  the parameter
  estimates found in step 2 of \algorithmref{algo:spectral-experts}
  satisfy the following with probability at least $1 - \delta$:
  \begin{align*}
  \|\hat \pi - \pi \|_{\infty} &\le \epsilon \\
  %&= O\left(\frac{k \pi_{\max}^{5/2}\| {M_3} \|_\op}{\sigma_k(M_2)^{5/2}} ~ \epsilon \right) \\
  \|\hat \beta_h - \beta_h\|_2 &\le \epsilon.
  %&= O\left( \frac{k \pi_{\max} \|M_2\|_\op^{1/2} \| {M_3} \|_\op}{ \sigma_k(M_2)^{5/2} }~ \epsilon \right),
  \end{align*}
  for all $h \in [k]$.
\end{lemma}

The proof follows by applying standard matrix perturbation results for
the whitening and unwhitening operators and 
\iftoggle{withappendix}{
can be found in \appendixref{sec:proofs:tensors}.
}{
has again been deferred to the supplementary material.
}
%\appendixref{sec:proofs}. 

%%%%%%%%%%%%%%%%%%%%%%%%%%%%%%%%%%%%%%%%%%%%%%%%%%%%%%%%%%%%
\subsection{Synthesis}
Together, these lemmas allow us to control the compound parameter error
and the recovery error. We now apply them in the proof of
\refthm{convergence}:

\begin{proof}[Proof of Theorem 1 (sketch)]
By \reflem{lowRank}, \reflem{lowRankLower} and \reflem{lowRankUpper}, we
can control the Frobenius norm of the error in the moments, which
directly upper bounds the operator norm: If $n \ge \max\{n_1, n_2\}$,
then
\begin{align}
  \|\hat M_p - M_p\|_\op = O\left( \lambda_n^{(p)} \sqrt{k} \sigmamin(\Sigma_p)^{-1} \right).
  %\|\hat M_p - M_p\|_\op = O\left( \sigma^3 L^3 R^6 k^{\frac12} \sigmamin(\Sigma_p)^{-1} \sqrt{\frac{\log^3 (1/\delta)}{n}} \right).
  %\|\hat M_p - M_p\|_F = O\left( \frac{\sigma^3 L^3 R^6 \sqrt{k} \log^3 (1/\delta)}{\sqrt{n} (\sigmamin(\Sigma_p) - d^p R^p \sqrt{\frac{p \log(d) \log(1/\delta)}{n}}} \right).
\end{align}

We complete the proof by applying \reflem{tensorPower} with the above
bound on $\|\hat M_p - M_p\|_\op$.

\end{proof}

\section{Empirical evaluation}
\label{sec:evaluation}

In the previous section, we showed that Spectral Experts provides a consistent
estimator.  %that converges a $\frac{1}{\sqrt n}$
In this section, we explore the empirical properties of our algorithm on simulated data.
Our main finding is that Spectral Experts alone attains higher parameter error than EM,
but this is not the complete story.
If we initialize EM with the estimates returned by Spectral Experts,
then we end up with much better estimates than EM from a random initialization.
%In this section, we will support these theoretical results with
%experiments on simulated data.
%We will see that in the finite sample
%regime, even when Spectral Experts does not converge to the optimal
%answer, it provides robust initialization for local methods like EM.

\subsection{Experimental setup}

\paragraph{Algorithms}

We experimented with three algorithms.
The first algorithm (Spectral) is simply the Spectral Experts.
We set the regularization strengths $\lambda_n^{(2)} = \frac{1}{10^{5} \sqrt{n}}$
and $\lambda_n^{(3)} = \frac{1}{10^{3} \sqrt{n}}$;
the algorithm was not very sensitive to these choices.
We solved the low-rank regression to estimate $M_2$ and
$M_3$ using an off-the-shelf convex optimizer, CVX~\cite{cvx}.
%ainterior point method an interior proximal subgradient descent
%algorithm.%\cite{candes11phaselift,tomioka2010estimation}. (Note sure about this citation and not required?)
%The algorithm was initialized with the regularized least squares
%solution and was run for 500 iterations or until convergence up to
%a tolerance of $10^{-3}$.
%Both regularization parameters, $\lambda_n^{(2)}$ and, $\lambda_n^{(3)}$ were annealed as $\frac{1}{n}$. 
The second algorithm (EM) is EM where the $\beta$'s are initialized from a standard normal
and $\pi$ was set to the uniform distribution plus some small perturbations.
We ran EM for 1000 iterations.
In the final algorithm (Spectral+EM),
we initialized EM with the output of Spectral Experts.

\paragraph{Data}

We generated synthetic data as follows:
First, we generated a vector $t$ sampled uniformly over the $b$-dimensional
unit hypercube $[-1,1]^b$.
Then, to get the actual covariates $x$, we applied a non-linear function of $t$
that conformed to the identifiability criteria discussed in
\sectionref{sec:algo}.
The true regression coefficients $\{\beta_h\}$ were drawn from a standard normal
and $\pi$ is set to the uniform distribution.
The observation noise $\epsilon$ is drawn from a normal with variance $\sigma^2$.
Results are presented below for $\sigma^2 = 0.1$, but we did not observe any
qualitatively different behavior for choices of $\sigma^2$ in the range
$[0.01, 0.4]$.  

As an example, one feature map we considered in the one-dimensional
setting $(b=1)$ was $x = (1, t, t^4, t^7)$. The data and the curves fit using
Spectral Experts, EM with random initialization and EM initialized with
the parameters recovered using Spectral Experts are shown in
\figureref{fig:curves}. We note that even on well-separated data such as
this, EM converged to the correct basin of attraction only 13\% of the time.

\begin{figure}[t]
  \centering
  \includegraphics[width=0.50\textwidth]{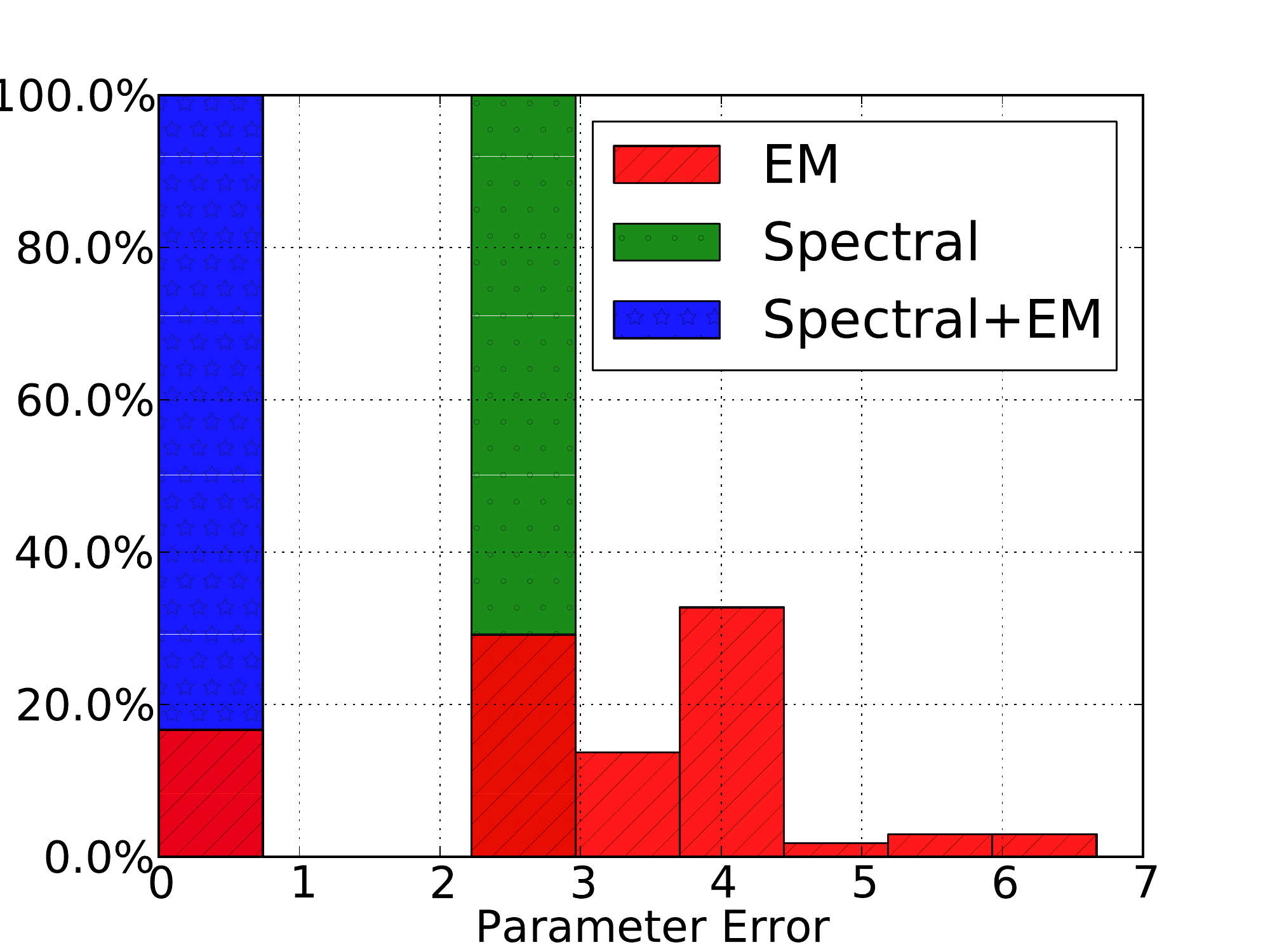}
  \caption{Histogram over recovery errors between the three algorithms when $b = 1, d = 4, k = 3, n = 500,000$.}
  \label{fig:hist}
\end{figure}

\begin{figure*}[p]
  \centering
  \subfigure[Spectral, Spectral+EM]{
    \includegraphics[width=0.50\textwidth]{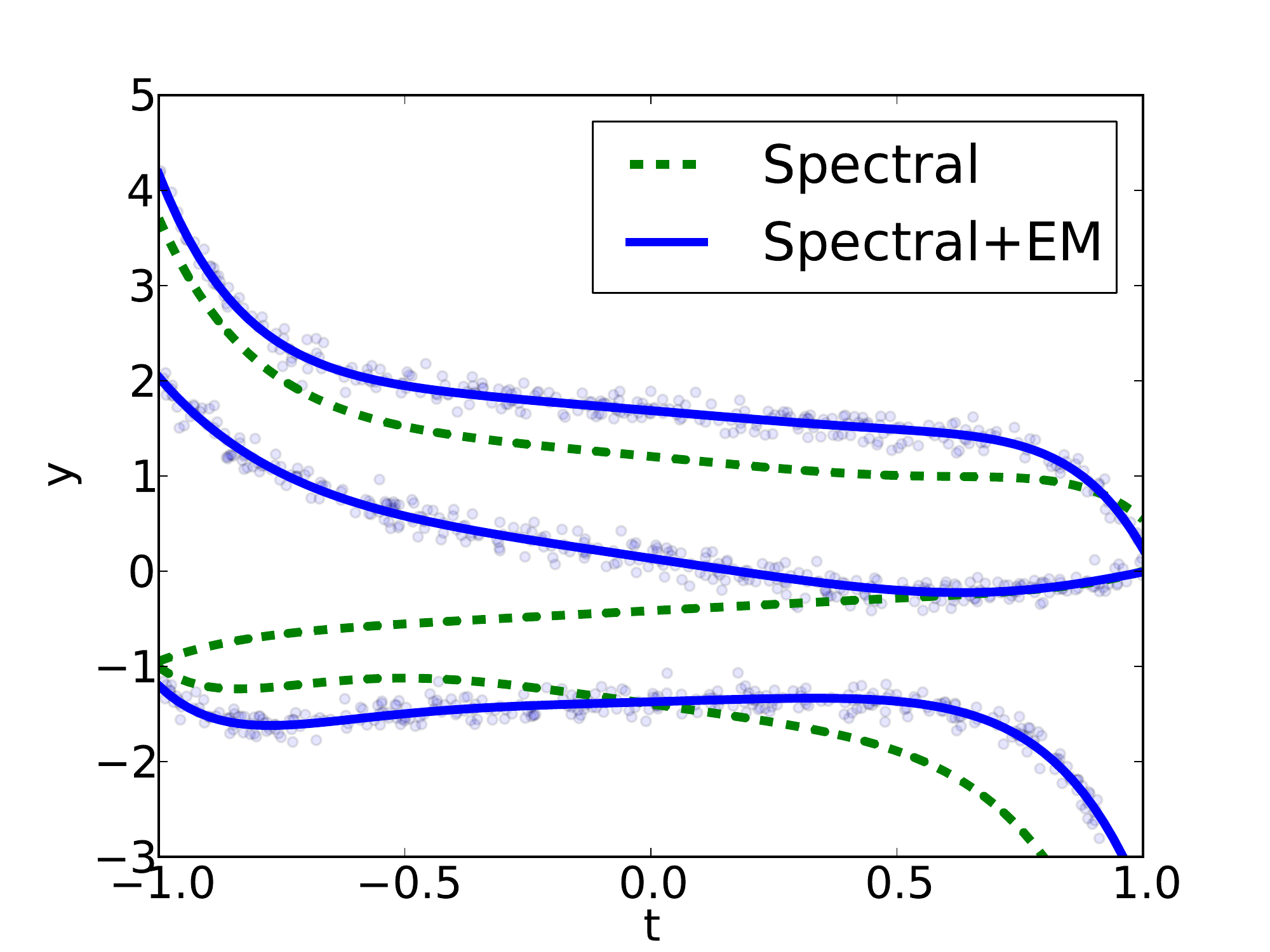}}
    \hspace{-2em}
  \subfigure[EM]{
    \includegraphics[width=0.50\textwidth]{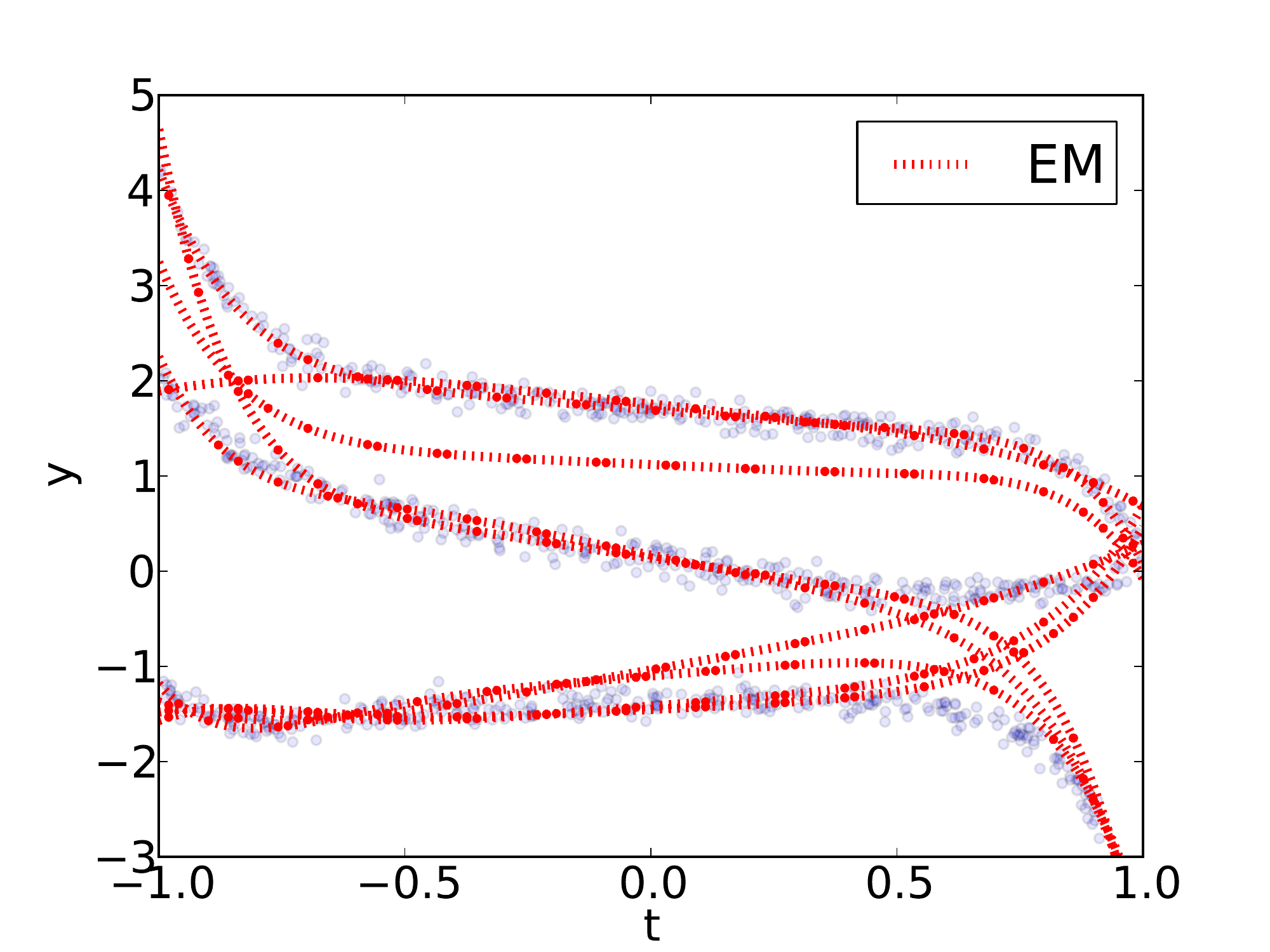}}
%  \subfigure[Spectral + EM]{
%    \includegraphics[width=0.34\textwidth]{figures/curves-vis/1-8-3-spem.png}}
  \caption{Visualization of the parameters estimated by Spectral Experts versus EM.
  (a) The dashed lines denote the solution recovered by Spectral Experts. While
  not a perfect fit, it provides an good initialization for EM to further improve the solution (solid lines).
  (b) The dotted lines show different local optima found by EM.}
  \label{fig:curves}
\end{figure*}

%Default values:
%$n = 10^6$
%$k = 5$

\begin{table*}[tbhp]
\caption{Parameter error $\|\theta^* - \hat \theta\|_F$ ($n = 500,000$)
as the number of base variables $b$, number of features $d$ and the number of components $k$ increases.
While Spectral by itself does not produce good parameter estimates, Spectral+EM improves
over EM significantly.
}
\label{tbl:parameter-recovery}
\vskip 0.15in
\begin{center}
\begin{small}
\begin{sc}

  \begin{tabular}{ r r r c c c }
\hline
\abovespace\belowspace
Variables ($b$) & Features ($d$) & Components ($k$) & Spectral & EM & Spectral + EM \\
\hline
\abovespace
%$\{1$, $ x_1$, $ x_1^4\}$ 
  1 & 4 & 2 & 2.45 $\pm$ 3.68 & 0.28 $\pm$ 0.82 & {\bf 0.17 $\pm$ 0.57} \\
% 1-5-2 2.451423 (+/- 3.686348) 0.176685 (+/- 0.579108) 0.007634 (+/- 0.001706) 0.341620 (+/- 0.584178) 0.010296 (+/- 0.007839)
%$\{1$, $ x_1$, $ x_2$, $ x_1^2 x_2^2\}$ 
2 & 5 & 2 & 1.38 $\pm$ 0.84 & {\bf 0.00 $\pm$ 0.00} & {\bf 0.00 $\pm$ 0.00} \\
% 2-3-2 1.383626 (+/- 0.840325) 0.001416 (+/- 0.000083) 0.003312 (+/- 0.000399) 0.454076 (+/- 0.758689) 0.003340 (+/- 0.000397)
  2 & 5 & 3 & 2.92 $\pm$ 1.71 & 0.43 $\pm$ 1.07 & {\bf 0.31 $\pm$ 1.02} \\
% 2-3-3 2.917581 (+/- 1.714054) 0.308649 (+/- 1.024135) 0.003885 (+/- 0.000442) 0.661009 (+/- 0.543162) 0.012455 (+/- 0.033185)
%$\{1$, $ x_1$, $ x_2$, $ x_1 x_2^3$, $ x_1^2 x_2^2$, $ x_1^3 x_2 \}$ 
  2 & 6 & 2 & 2.33 $\pm$ 0.67 & 0.63 $\pm$ 1.29 & {\bf 0.01 $\pm$ 0.01} \\
% 2-4-2 2.333000 (+/- 0.674578) 0.004067 (+/- 0.001391) 0.001167 (+/- 0.001096) 0.592667 (+/- 0.478518) 0.001200 (+/- 0.001143)
%\belowspace
%  2 & 6 & 5 & 6.78 $\pm$ 2.18 & 2.29 $\pm$ 1.79 & {\bf 1.77 $\pm$ 1.89} \\

%  2 & 5 & 3 & 1.87 $\pm$ 1.20 & {\bf 0.33 $\pm$ 0.96} & 0.35 $\pm$ 1.23 \\
% 2 & 6 & 5 & 5.27 $\pm$ 2.32 & 1.80 $\pm$ 1.80 & {\bf 1.51 $\pm$ 1.77} \\
%$\{1$, $ x_1$, $ x_2$, $ x_1 x_2^3$, $ x_1^2 x_2^2$, $ x_1^3 x_2$, $ x_1^3 x_2^4$, $ x_1^4 x_2^3 \}$ 
% 2 & 8 & 7 & 8.42 $\pm$ 1.66 & 7.41 $\pm$ 2.99 & {\bf 7.31 $\pm$ 2.47} \\
%$\{1$, $ x_1$, $ x_2$, $ x_3$, $ x_1 x_2 x_3^2$, $ x_1 x_2^2 x_3$, $ x_1^2 x_2 x_3$, $ x_1^2 x_2^2$, $ x_1^2 x_3^2$, $ x_2^2 x_3^2$, $ x_3^2\}$
% 3 & 10 & 3 & 3.78 $\pm$ 0.90 & 0.67 $\pm$ 1.49 & {\bf 0.51 $\pm$ 1.12} \\
%$\{1$, $ x_1$, $ x_2$, $ x_3$, $ x_1 x_2 x_3^2$, $ x_1 x_2^2 x_3$, $ x_1^2 x_2 x_3$, $ x_1^2 x_2^2$, $ x_1^2 x_3^2$, $ x_2^2 x_3^2$, $ x_3^2\}$
% 3 & 10  & 7 & 9.97 $\pm$ 3.22 & 1.43 $\pm$ 1.97 & {\bf 1.42 $\pm$ 2.17} \\

\hline

\end{tabular}
\end{sc}
\end{small}
\end{center}
\vskip -0.1in
\end{table*}

% \todo{Dataset 2: $b = 30$, $p = 1$. I'm unclear as to what we can show
% on this sort of data set. EM works extremely well, and the spectral
% methods do not converge easily.}

\begin{figure*}[tbhp]
  \centering
  \subfigure[Well-specified data]{
    \includegraphics[width=0.50\textwidth]{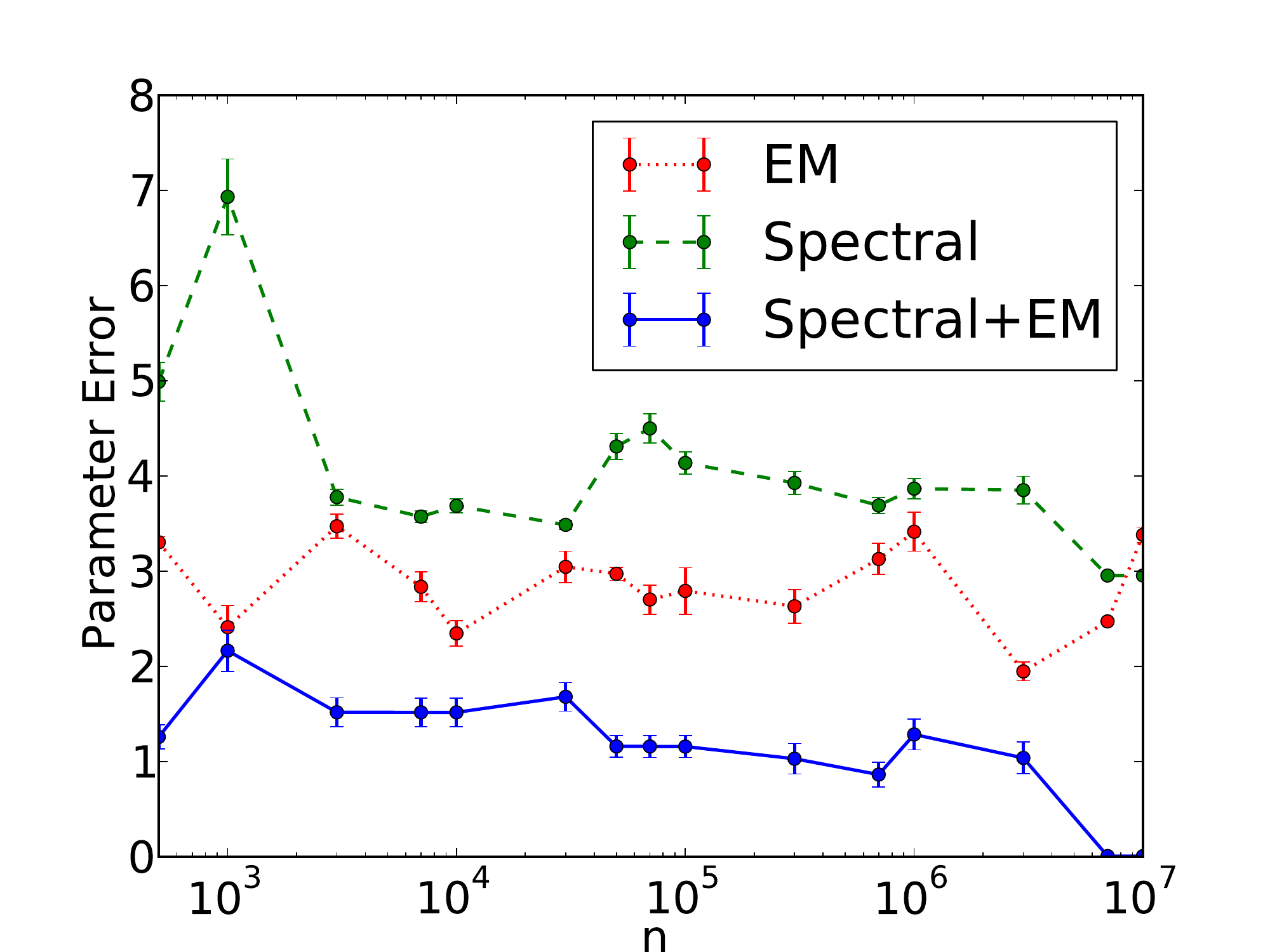}
  }
    \hspace{-2em}
  \subfigure[Misspecified data]{
    \includegraphics[width=0.50\textwidth]{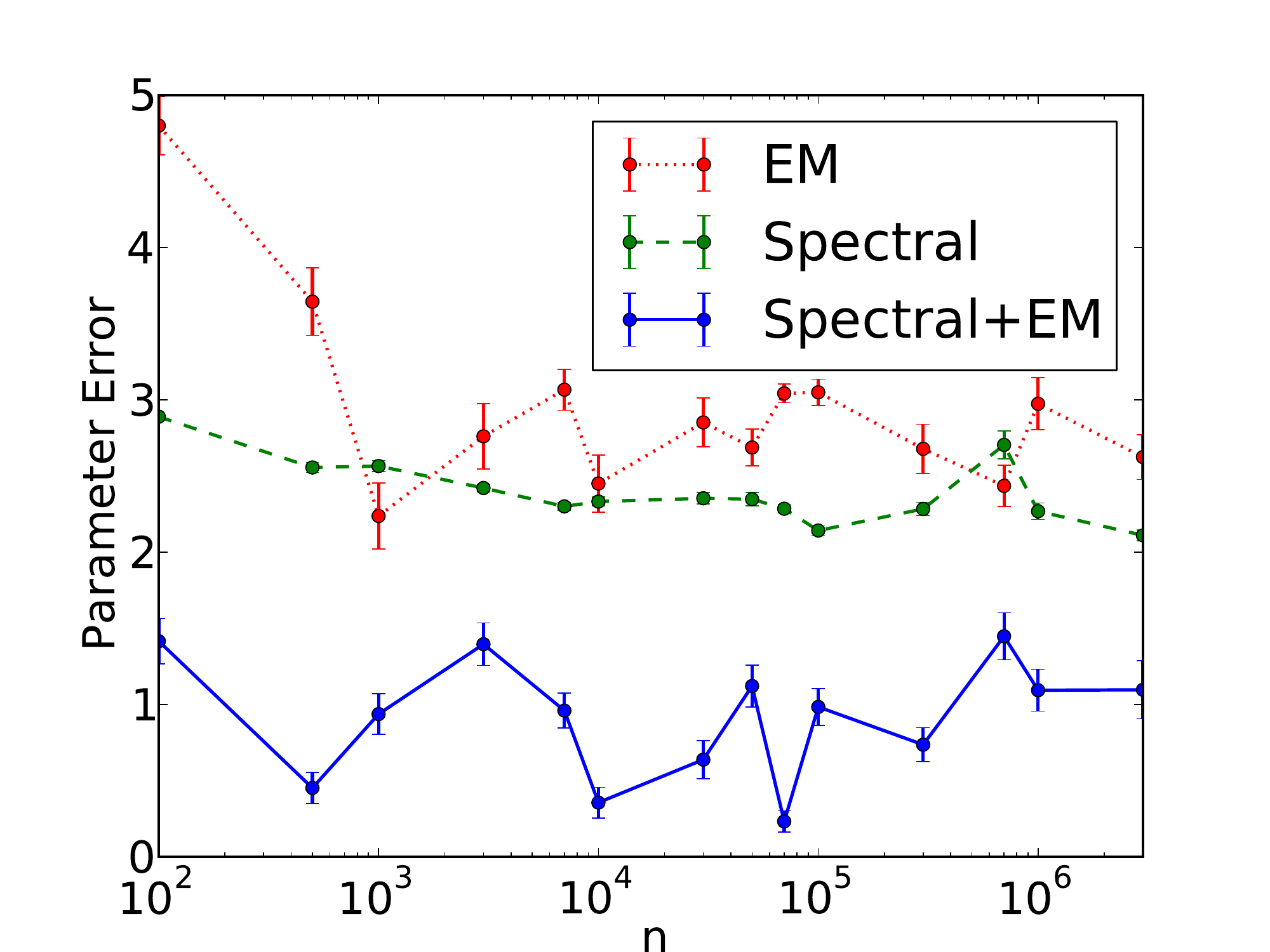}
  }
  \caption{Learning curves: parameter error as a function of the number of samples $n$ ($b = 1, d = 5, k = 3$).}
  \label{fig:vs-n}
\end{figure*}

\subsection{Results}

\begin{table*}[tbhp]
\caption{Parameter error $\|\theta^* - \hat \theta\|_F$ when the data is misspecified ($n = 500,000$).
Spectral+EM degrades slightly, but still outperforms EM overall.
}
\label{tbl:parameter-recovery-mis}
\vskip 0.15in
\begin{center}
\begin{small}
\begin{sc}

  \begin{tabular}{ r r r c c c }
\hline
\abovespace\belowspace
Variables ($b$) & Features ($d$) & Components ($k$) & Spectral & EM & Spectral + EM \\
\hline
\abovespace
 1 & 4 & 2 &  1.70 $\pm$ 0.85 & 0.29 $\pm$ 0.85 &  {\bf 0.03 $\pm$ 0.09} \\
% 1-5-2 1.707638 (+/- 0.849695) 0.028238 (+/- 0.093263) 0.008125 (+/- 0.002117) 0.657658 (+/- 0.769400) 0.008112 (+/- 0.002153)
%\belowspace
 2 & 5 & 3 &  1.37 $\pm$ 0.85 & 0.44 $\pm$ 1.12 &  {\bf 0.00 $\pm$ 0.00} \\
% 2-3-2 1.373689 (+/- 0.849721) 0.002074 (+/- 0.000116) 0.004720 (+/- 0.000000) 0.501348 (+/- 0.656840) 0.004702 (+/- 0.000009)
 2 & 6 & 5 &  9.89 $\pm$ 4.46 & {\bf 2.53 $\pm$ 1.77} &  2.69 $\pm$ 1.83 \\
 2 & 8 & 7 & 23.07 $\pm$ 7.10 & 9.62 $\pm$ 1.03 &  {\bf 8.16 $\pm$ 2.31}  \\
\hline

\end{tabular}
\end{sc}
\end{small}
\end{center}
\vskip -0.1in
\end{table*}

\tableref{tbl:parameter-recovery} presents the Frobenius norm of the
difference between true and estimated parameters for the model, averaged
over 20 different random instances for each feature set and 10 attempts
for each instance. The experiments were run using $n = 500,000$ samples.

One of the main reasons for the high variance is the variation across
random instances; some are easy for EM to find the global minima and
others more difficult. In general, while Spectral Experts did not
recover parameters by itself extremely well, it provided a good initialization for
EM.

To study the stability of the solutions returned by Spectral Experts,
consider the histogram in \figureref{fig:hist}, which shows the recovery
errors of the algorithms over 170 attempts on a dataset with $b = 1, d = 4,
k = 3$. Typically, Spectral Experts returned a stable solution.
When these parameters were close enough to the true parameters, we found
that EM almost always converged to the global optima. Randomly
initialized EM only finds the true parameters a little over 10\% of the
time and shows considerably higher variance. 

\paragraph{Effect of number of data points}

In \figureref{fig:vs-n}, we show how the recovery error varies as we get
more data. Each data point shows the mean error over 10 attempts, with
error bars. We note that the recovery performance of EM does not
particularly improve; this suggests that EM continues to get stuck in
a local optima. The spectral algorithm's error decays slowly, and as it
gets closer to zero, EM initialized at the spectral parameters finds the
true parameters more often as well. This behavior highlights the
trade-off between statistical and computational error. 

\paragraph{Misspecified data}

To evaluate how robust the algorithm was to model mis-specification, we
removed large contiguous sections from $x \in [-0.5,-0.25] \cup
[0.25,0.5]$ and ran the algorithms again.
\tableref{tbl:parameter-recovery-mis} reports recovery errors in this
scenario. The error in the estimates grows larger for higher $d$.

\section{Conclusion}
\label{sec:conclusion}

In this paper, we developed a computationally efficient and statistically
consistent estimator for mixture of linear regressions.
Our algorithm, Spectral
Experts, regresses on higher-order powers of the data with a regularizer that
encourages low rank structure, followed by tensor factorization to recover the
actual parameters.  Empirically, we found Spectral Experts 
to be an excellent initializer for EM.
%yielding significantly more accurate and stable results than EM.

\paragraph{Acknowledgements}
We would like to thank Lester Mackey for his fruitful suggestions and
the anonymous reviewers for their helpful comments.

\bibliography{ref,pliang}
\bibliographystyle{icml2013}

\iftoggle{withappendix}{
\appendix
\onecolumn
\section{Proofs: Regression}
\label{sec:proofs:regression}
\setcounter{lemma}{0}

%\iftoggle{withappendix}{}{

Let us review the regression problem set up in 
\iftoggle{withappendix}{\refsec{algo}}{\cite[Section 3]{ChagantyLiang2013}}. 
We assume we are given data $(x_i,y_i) \in \sD_p$
generated by the following process,
\begin{align*}
  y_i &= \innerp{M_p}{x_i\tp{p}} + b_p + \eta_p(x_i),
\end{align*}
where $M_p = \sum_{h=1}^k \pi_h \beta_h\tp{p}$, the expected value of $\beta_h\tp{p}$, $b_p$ is an estimable bias and $\eta_p(x)$ is zero mean noise. In particular, for $p \in \{1,2,3\}$,
we showed that $b_p$ and $\eta_p(x)$ were,
\begin{align}
  b_1 &= 0 \nonumber \\
  b_2 &= \E[\epsilon^2] \nonumber \\
  b_3 &= 2\E[\epsilon]\innerp{\hat M_1}{x} + \E[\epsilon^3] \nonumber \\
  \eta_1(x) &= \innerp{\beta_h - M_1}{x} + \epsilon \label{eqn:app:eta1} \\
  \eta_2(x) &= \innerp{\beta_h\tp{2} - M_2}{x\tp{2}} + 2 \epsilon \innerp{\beta_h}{x} + (\epsilon^2 - \E[\epsilon^2]) \label{eqn:app:eta2}\\
  \eta_3(x) &= \innerp{\beta_h\tp{3} - M_3}{x\tp{3}}
        + 3 \epsilon \innerp{\beta_h\tp{2}}{x\tp{2}} 
        + 3(\epsilon^2 \innerp{\beta_h}{x} - \E[\epsilon^2] \innerp{M_1}{x})
        + (\epsilon^3 - \E[\epsilon^3]). \label{eqn:app:eta3}
\end{align}
%We assume that $\|x_i\| \le R$, $\| \beta_h \| \le L$ and $|\epsilon| \le S$.

We then defined the observation operator $\opX_p(M_p) : \Re^{d\tp{p}} \to \Re^{n}$,
\begin{align*}
\opX_p(M_p; \sD_p)_i &\eqdef \innerp{M_p}{x\tp{p}_i},
\end{align*}
for $(x_i, y_i) \in \sD_p$. This let us succinctly represent the
low-rank regression problem for $p = 2,3$ as follows,
\begin{align*}
  \hat M_p &=
  \arg\min_{M_p \in \Re^{d\tp{p}}} \frac{1}{2n} \| y - b_p - \opX_p(M_p; \sD_p) \|^2_2 + \lambda_p \|M_p\|_*.
\end{align*}

Let us also recall the adjoint of the observation operator, $\opX_p^* : \Re^{n} \to \Re^{d^p}$,
\begin{align*}
  \opX_p^*(\eta_p; \sD_p) &= \sum_{x \in \sD_p} \eta_p(x) x\tp{p},
\end{align*}
where we have used $\eta_p$ to represent the vector $\left[\eta_p(x)\right]_{x \in \sD_p}$. 

\citet{Tomioka2011} showed that error in the estimated $\hat M_p$ can be
bounded as follows;

\begin{lemma}[\citet{Tomioka2011}, Theorem 1]
\label{lem:app:lowRank}
Suppose there exists a restricted strong convexity constant $\kappa(\opX_p)$ such that
$$\frac{1}{2n} \| \opX_p( \Delta )\|_2^2 \ge \kappa(\opX_p) \|\Delta\|^2_F \quad \text{and} \quad
\lambda_n \ge \frac{\|\opX_p^*(\eta_p)\|_\op}{n}.$$
Then the error of $\hat M_p$ is bounded as follows:
$\| \hat M_p - M_p^* \|_F \le \frac{\lambda_n \sqrt{k}}{\kappa(\opX_p)}$.
\end{lemma}
%} % endif

%%%%%%%%%%%%%%%

In this section, we will derive an upper bound on $\kappa(\opX_p)$ and
a lower bound on $\frac{1}{n} \| \opX_p^*(\eta_p) \|_\op$, allowing us to apply \reflem{lowRank} in the particular context of our noise setting.

\begin{lemma}[Lower bound on restricted strong convexity]
\label{lem:app:lowRankLower}
Let $\Sigma_p \eqdef \E[\cvec(x\tp{p})\tp{2}]$.
If $$n \ge \frac{16 (p!)^2 R^{4p}}{\sigmamin(\Sigma_p)^2} \left(1 + \sqrt{\frac{\log(1/\delta)}{2}}\right)^2,$$
then, with probability at least $1-\delta$,
$$\kappa(\opX_p) \ge \frac{\sigmamin(\Sigma_p)}{2}.$$
\end{lemma}

\begin{proof}
  Recall that $\kappa(\opX_p)$ is defined to be a constant such that the following inequality holds, 
  $$\frac{1}{n} \|\opX_p(\Delta)\|_2^2 \ge \kappa(\opX_p) \|\Delta\|^2_F,$$
where
$$\|\opX_p(\Delta)\|_2^2 = \sum_{(x,y) \in \sD_p} \innerp{\Delta}{x\tp{p}}^2.$$
To proceed, we will unfold the tensors $\Delta$ and $x\tp{p}$ to get a lower bound in terms of $\|\Delta\|^2_F$. This will allow us to choose an appropriate value for $\kappa(\opX_p)$ that will hold with high probability.

First, note that $x\tp{p}$ is symmetric, and thus
$\innerp{\Delta}{x\tp{p}} = \innerp{\cvec{\Delta}}{\cvec{x\tp{p}}}$.
This allows us to simplify $\|\opX_p(\Delta)\|_2^2$ as follows,
\begin{align*}
  \frac{1}{n} \|\opX_p(\Delta)\|_2^2 
    &= \frac{1}{n} \sum_{(x,y) \in \sD_p} \innerp{\Delta}{x\tp{p}}^2 \\
    &= \frac{1}{n} \sum_{(x,y) \in \sD_p} \innerp{\cvec(\Delta)}{\cvec(x\tp{p})}^2 \\
    &= \frac{1}{n} \sum_{(x,y) \in \sD_p} \trace( \cvec(\Delta)\tp{2} \cvec(x\tp{p})\tp{2} ) \\
    &= \trace\left( \cvec(\Delta)\tp{2} \frac{1}{n} \sum_{(x,y) \in \sD_p} \cvec(x\tp{p})\tp{2} \right).
\end{align*}
Let $\hat\Sigma_p \eqdef \frac{1}{n}
\sum_{(x,y) \in \sD_p} \cvec(x\tp{p})\tp{2}$, so that $\frac{1}{n}
\|\opX_p(\Delta)\|_2^2 = \trace(\cvec(\Delta)\tp{2} \hat\Sigma_p)$. 
For symmetric $\Delta$
, $\|\cvec(\Delta)\|_2 = \|\Delta\|_F$\footnote{
The $\Delta$ correspond to residuals $\hat M_p - M_p$, which can easily be shown to always be symmetric.}.
Then, we have 
\begin{align*}
\frac{1}{n} \|\opX_p(\Delta)\|_2^2 
  &= \trace(\cvec(\Delta)\tp{2} \hat\Sigma_p) \\
  &\ge \sigmamin(\hat\Sigma_p) \|\Delta\|_F^2.
\end{align*}
By Weyl's theorem, $$\sigmamin(\hat\Sigma_p) \ge
\sigmamin(\Sigma_p) - \|\hat\Sigma_p - \Sigma_p\|_\Lop.$$
Since $\|\hat\Sigma_p - \Sigma_p\|_\Lop \le \|\hat\Sigma_p - \Sigma_p\|_{F}$,
it suffices to show that the empirical covariance concentrates in Frobenius norm.
Applying \reflem{conc-norms}, with
probability at least $1 - \delta$, $$\| \hat\Sigma_p - \Sigma_p \|_F
\le \frac{2 \|\Sigma_p\|_F}{\sqrt n} \left( 1 + \sqrt{\frac{\log(1/\delta)}{2}} \right).$$
Now we seek to control $\|\Sigma_p\|_F$.
Since $\|x\|_2 \le R$, we can use the
bound $$\| \Sigma_p \|_F \le p! \| \vvec(x\tp{p})\tp{2} \|_F \le p! R^{2p}.$$

Finally, $\|\hat\Sigma_p - \Sigma_p\|_\op \le \sigmamin(\Sigma_p)/2$ with probability at least $1 - \delta$ if,
$$n \ge \frac{16 (p!)^2 R^{4p}}{\sigmamin(\Sigma_p)^2} \left(1 + \sqrt{\frac{\log(1/\delta)}{2}}\right)^2.$$

% To do this, we first apply Hoeffding's inequality elementwise.
% Since $\|x\|_2 \le R$, we have that for each element $(i,j)$,
% $|(\hat\Sigma_p)_{ij} - (\Sigma_p)_{ij}| = O(R^{2p}\sqrt{\frac{\log (1/\delta)}{n}})$.
% Applying the union bound over the $d^{2p}$ elements of $\hat\Sigma_p - \Sigma_p$,
% we have that the max norm is bounded:
% $\|\hat\Sigma_p - \Sigma_p\|_\text{\rm max} = O(R^{2p} \sqrt{\frac{p \log(d) \log (1/\delta)}{n}})$.
% The max norm times $d^p$ upper bounds the Frobenius norm, which upper bounds the operator norm, so we have that
% $\|\hat\Sigma_p - \Sigma_p\|_\text{\rm op} = O(d^p R^{2p} \sqrt{\frac{p \log(d) \log (1/\delta)}{n}})$.
% Using the fact that $\sigmamin(\hat\Sigma_p) \ge \sigmamin(\Sigma_p) - \|\hat\Sigma_p - \Sigma_p\|_\text{\rm op}$
% yields the result.
\end{proof}

\begin{lemma}[Upper bound on adjoint operator]
\label{lem:app:lowRankUpper}
With probability at least $1-\delta$, the following holds,
\begin{align*}
  \frac{1}{n} \|\opX_1^*(\eta_1)\|_\op
      &\le 2 \frac{R (2LR + S)}{\sqrt{n}} \left( 1 + \sqrt{\frac{\log(3/\delta)}{2}} \right) \\
  \frac{1}{n}  \|\opX_2^*(\eta_2)\|_\op 
      &\le 2 \frac{(4L^2 R^2 + 2 S L R + 4S^2)R^2}{\sqrt{n}} \left( 1 + \sqrt{\frac{\log(3/\delta)}{2}} \right) \\
  \frac{1}{n} \|\opX_3^*(\eta_3)\|_\op 
      &\le 2 \frac{(8L^3 R^3 + 3 L^2 R^2 S + 6 L R S^2 + 2S^3) R^3}{\sqrt{n}} \left( 1 + \sqrt{\frac{\log(6/\delta)}{2}} \right) \\
  &\quad + 3 R^4 S^2 \left(\frac{128R(2LR+S)}{\sigmamin(\Sigma_1) \sqrt{n}} \left(1 + \sqrt{\frac{\log(6/\delta)}{2}} \right) \right).
\end{align*}
\end{lemma}

It follows that, with probability at least $1-\delta$,
\begin{align*}
  \frac{1}{n} \|\opX_p^*(\eta_p)\|_\op
  &= O\left( L^p S^p R^{2p} \sigmamin(\Sigma_1)^{-1} \sqrt{\frac{\log(1/\delta)}{n}} \right),
\end{align*}
for each $p \in \{1,2,3\}$.

\begin{proof}
Let $\hat\E_p[f(x,\epsilon,h)]$ denote the empirical expectation over
the examples in dataset $\sD_p$ (recall the $\sD_p$'s are independent to
simplify the analysis).  By definition,
$$\frac1n \|\opX_p^*(\eta_p)\|_\op = \left\| \hat\E_p \left[\eta_p(x) x\tp{p} \right] \right\|_\op $$
for $p \in \{1,2,3\}$. To proceed, we will bound each $\eta_p(x)$
\iftoggle{withappendix}{}{
, defined in \refeqn{app:eta1}, \refeqn{app:eta2} and \refeqn{app:eta3},} and use \reflem{conc-norms} to bound $\|
\hat\E_p[\eta_p(x) x\tp{p}] \|_F$. The Frobenius norm to bounds the
operator norm, completing the proof.

%composed of several zero-mean random variables, and
%since $\|A + B\|_\op \le \|A\|_\op + \|B\|_\op$, it suffices to consider
%each term in turn. 

\paragraph{Bounding $\eta_p(x)$.}
Using the assumptions that $\|\beta_h\|_2 \le L$, $\|x\|_2 \le R$ and
$|\epsilon| \le S$, it is easy to bound each $\eta_p(x)$,
\begin{align*}
  \eta_1(x) &= \innerp{\beta_h - M_1}{x} + \epsilon \\
            &\le \|\beta_h - M_1\|_2 \|x\|_2 + |\epsilon| \\
            &\le 2LR + S \\
  \eta_2(x) 
    &= \innerp{\beta_h\tp{2} - M_2}{x\tp{2}} + 2 \epsilon \innerp{\beta_h}{x} + (\epsilon^2 - \E[\epsilon^2]) \\
    &\le \|\beta_h\tp{2} - M_2\|_F \|x\tp{2}\|_F + 2 |\epsilon| \|\beta_h\|_2\|x\|_2 + |\epsilon^2 - \E[\epsilon^2]| \\
    &\le (2L)^2 R^2 + 2 S L R + (2S)^2 \\
  \eta_3(x) &= \innerp{\beta_h\tp{3} - M_3}{x\tp{3}}
        + 3 \epsilon \innerp{\beta_h\tp{2}}{x\tp{2}} \\
        &\quad + 3\left(\epsilon^2 \innerp{\beta_h}{x} - \E[\epsilon^2] \innerp{\hat M_1}{x}\right)
        + (\epsilon^3 - \E[\epsilon^3]) \\
  &\le \|\beta_h\tp{3} - M_3\|_F\|x\tp{3}\|_F
        + 3 |\epsilon| \|\beta_h\tp{2}\|_F \|x\tp{2}\|_F  \\
        &\quad + 3 \left( |\epsilon^2|~\|\beta_h\|_F\|x\|_F + \left|\E[\epsilon^2]\right|~\|\hat M_1\|_2\|x\|_2 \right)
        + |\epsilon^3| + \left| \E[\epsilon^3] \right| \\
  &\le (2L)^3 R^3 + 3 S L^2 R^2 + 3 ( S^2 L R + S^2 L R ) + 2S^3.
\end{align*}
We have used inequality $\|M_1 - \beta_h\|_2 \le 2L$ above. 

\paragraph{Bounding $\left\| \hat\E[\eta_p(x)x\tp{p}] \right\|_F$.}
We may now apply the above bounds on $\eta_p(x)$ to bound $\|\eta_p(x) x\tp{p}\|_F$, using the fact that $\|c X\|_F \le c\|X\|_F$.
By \reflem{conc-norms}, each of the following holds with probability at least $1-\delta_1$,
\begin{align*}
    \left\|\hat\E_1[\eta_1(x) x] \right\|_2
    &\le 2 \frac{R (2LR + S)}{\sqrt{n}} \left( 1 + \sqrt{\frac{\log(1/\delta_1)}{2}} \right) \\
  \left\|\hat\E_2[\eta_2(x) x\tp{2}] \right\|_F
      &\le 2 \frac{(4L^2 R^2 + 2 S L R + 4S^2)R^2}{\sqrt{n}} \left( 1 + \sqrt{\frac{\log(1/\delta_2)}{2}} \right) \\
  \left\|\hat\E_3[\eta_3(x) x\tp{3}] - \E[\eta_3(x) x\tp{3} \mid x] \right\|_F
      &\le 2 \frac{(8L^3 R^3 + 3 L^2 R^2 S + 6 L R S^2 + 2S^3) R^3}{\sqrt{n}} \left( 1 + \sqrt{\frac{\log(1/\delta_3)}{2}} \right).
\end{align*}

Recall that $\eta_3(x)$ does not have zero mean, so we must bound the bias:
\begin{align*}
  \| \E[\eta_3(x) x\tp{3} \mid x] \|_F &= \|3 \E[\epsilon^2] \innerp{M_1 - \hat M_1}{x} x\tp{3} \|_F \\
    &\le 3 \E[\epsilon^2] \|M_1 - \hat M_1\|_2 \|x\|_2 \|x\tp{3}\|_F.
\end{align*}
Note that in all of this, both $\hat M_1$ and $M_1$ are treated as
constants. 
Further, by applying \reflem{app:lowRank} to $M_1$, we have a bound on $\|M_1 - \hat
M_1\|_2$; with probability at least $1-\delta_3$,
\begin{align*}
  \| M_1 - \hat M_1 \|_2
  &\le \frac{32 \lambda^{(1)}_n}{\kappa(\opX_1)} \\
  &\le 32 \frac{2R(2LR+S)}{\sqrt{n}}\left(1 + \sqrt{\frac{\log(1/\delta_3)}{2}}\right) \frac{2}{\sigmamin(\Sigma_1)}.
\end{align*}

So, with probability at least $1 - \delta_3$,
\begin{align*}
  \| \E[\eta_3(x) x\tp{3} \mid x] \|_F
  &\le 3 R^4 S^2 \left(\frac{128 R(2LR+S)}{\sigmamin(\Sigma_1) \sqrt{n}} \left(1 + \sqrt{\frac{\log(1/\delta_3)}{2}} \right) \right).
\end{align*}

Finally, taking $\delta_1 = \delta/3, \delta_2 = \delta/3, \delta_3
= \delta/6$, and taking the union bound over the bounds for $p \in
\{1,2,3\}$, we get our result.
\end{proof}

%%%%%%%%%%%%

\section{Proofs: Tensor Decomposition}
\label{sec:proofs:tensors}

Once we have estimated the moments from the data through regression, we apply the robust tensor eigen-decomposition algorithm to recover the parameters, $\beta_h$ and $\pi$. However, the algorithm is guaranteed to work only for symmetric matrices with (nearly) orthogonal eigenvectors, so, as a first step, we will need to whiten the third-order moment tensor using the second moments. We then apply the tensor decomposition algorithm to get the eigenvalues and eigenvectors. Finally, we will have to undo the transformation by applying an un-whitening step. In this section, we present error bounds for each step, and combine them to prove the following lemma,
\begin{lemma}[Tensor Decomposition with Whitening]
  \label{lem:app:tensorPower} 
  Let $M_2 = \sum_{h=1}^{k} \pi_h \beta_h\tp{2}$,
  $M_3 = \sum_{h=1}^{k} \pi_h \beta_h\tp{3}$.
  Let $\aerr{M_2} \eqdef \|\hat M_2 - M_2\|_\op$ and
  $\aerr{M_3} \eqdef \|\hat M_3 - M_3\|_\op$ both be such that,

\begin{align*}
  \max\{\aerr{M_2}, \aerr{M_3}\} &\le \min\Bigg\{
    \frac{\sigma_k(M_2)}{2},
      \left(\frac{15 k \pi_{\max}^{5/2}
      \left(24 \frac{\| {M_3} \|_\op}{\sigma_k(M_2)} + 2\sqrt{2} \right)}{2\sigma_k(M_2)^{3/2}} \right)^{-1} \epsilon,\\
     &\quad 
      \left(
      4\sqrt{3/2} \|M_2\|_\op^{1/2} \sigma_k(M_2)^{-1} + 8 k \pi_{\max} 
      \|M_2\|_\op^{1/2} \sigma_k(M_2)^{-3/2}
      \left(24 \frac{\| {M_3} \|_\op}{\sigma_k(M_2)} + 2\sqrt{2} \right) \right)^{-1} \epsilon
  \Bigg\}
\end{align*}
  
%  less than $\frac{1}{2 \sigma_k(M_2)}$
%  \begin{align*}
%    \frac{2\sigma_k(M_2)^{3/2}}{15 k \pi_{\max}^{5/2}
%    \left(24 \frac{\| {M_3} \|_\op}{\sigma_k(M_2)} + 2\sqrt{2} \right)}~ \epsilon
%  \end{align*}
%  and
%\begin{align*}
%    \left(
%    4\sqrt{3/2} \|M_2\|_\op^{1/2} \sigma_k(M_2)^{-1} + 8 k \pi_{\max} 
%    \|M_2\|_\op^{1/2} \sigma_k(M_2)^{-3/2}
%    \left(24 \frac{\| {M_3} \|_\op}{\sigma_k(M_2)} + 2\sqrt{2} \right) \right)^{-1}~ \epsilon,
%\end{align*}
for some $\epsilon < \frac{1}{2\sqrt{\pi_{\max}}}$.
%\begin{align*}
%  \epsilon &\le 
%    \min\Bigg\{
%    \left(
%    4\sqrt{3/2} \|M_2\|_\op^{1/2} \sigma_k(M_2)^{-1} + 8 k \pi_{\max} 
%    \|M_2\|_\op^{1/2} \sigma_k(M_2)^{-3/2}
%    \left(24 \frac{\| {M_3} \|_\op}{\sigma_k(M_2)} + 2\sqrt{2} \right) \right) \frac{\sigma_k(M_2)}{2}, \\
%  &\quad 
%    \frac{10 k \pi_{\max}^{5/2}
%    \left(24 \frac{\| {M_3} \|_\op}{\sigma_k(M_2)} + 2\sqrt{2} \right)}
%    {3\sigma_k(M_2)^{3/2}}~ \frac{\sigma_k(M_2)}{2}, 
%    \frac{1}{2\sqrt{\pi_{\max}}} \Bigg\}.
%\end{align*}

  Then, there exists a permutation of indices such that  the parameter
  estimates found in step 2 of 
  \iftoggle{withappendix}{%
  \algorithmref{algo:spectral-experts}
  }{%
  \citet[Algorithm 1]{ChagantyLiang2013}
  }
  satisfy the following with probability at least $1 - \delta$,
  \begin{align*}
  \|\hat \pi - \pi \|_{\infty} &\le \epsilon \\
  \|\hat \beta_h - \beta_h\|_2 &\le \epsilon.
  \end{align*}
  for all $h \in [k]$.
\end{lemma}

\begin{proof}
We will use the general notation, $\aerr{X} \eqdef \|\hat X - X\|_\Lop$
to represent the error of the estimate, $\hat X$, of $X$ in the operator
norm. 

Through the course of the proof, we will make some assumptions on errors
that allow us simplify portions of the expressions. At the end of
the proof, we will collect these conditions together to state the
assumptions on $\epsilon$ above.

\paragraph{Step 1: Whitening}
Much of this matter has been presented in \citet[Lemma 11, 12]{hsu13spherical}. We present our own version for completeness.

Let $W$ and $\hat W$ be the whitening matrices for $M_2$ and $\hat M_2$
respectively. Also define $\Winv$ and $\Whinv$ to be their
pseudo-inverses.

We will first show that the whitened tensor $T = M_3(W,W,W)$ is symmetric with orthogonal
eigenvectors. Recall that $M_2 = \sum_h \pi_h \beta_h\tp{2}$, so,
\begin{align*}
  I 
    &= W^T M_2 W\\
    &= \sum_h \pi_h W^T \beta_h\tp{2} W\\
    &= \sum_h (\underbrace{\sqrt{\pi_h} W^T \beta_h}_{v_h})\tp{2}.
\end{align*}
Thus $W \beta_h = \frac{v_h}{\sqrt{\pi_h}}$,
where $v_h$ form an orthonormal
basis. Applying the same whitening transform to $M_3$, we get, 
\begin{align*}
  M_3 &= \sum_h \pi_h \beta_h\tp{3} \\
  M_3(W,W,W) &= \sum_h \pi_h (W^T \beta_h)\tp{3} \\
  &= \sum_h \frac{1}{\sqrt{\pi_h}} v_h\tp{3}.
\end{align*}
Consequently, $T$ has an orthogonal decomposition with eigenvectors $v_h$ and eigenvalues $1/\sqrt{\pi_h}$.

Let us now study how far $\hat T = \hat M_3(\hat W, \hat W, \hat W)$ differs from $T$, in terms of the
errors of $M_2$ and $M_3$, following \citet{AnandkumarGeHsu2012}. We note that while $\hat T$ is also symmetric, it may not have an orthogonal decomposition. 
To do so, we use the triangle inequality to
break the difference into a number of simple terms, that differ in exactly one element. We will then apply $\|M_3(W,W,W-\hat W)\|_\op \le \|M_3\|_\op \|W\|_\op \|W\|_\op \|W - \hat W\|_\op$.

\begin{align*}
  \aerr{T} &= \|M_3(W,W,W) - \hat M_3(\hat W,\hat W, \hat W)\|_\op \\
           &\le 
           \| {M_3}(W,W,W) - {M_3}(W,W,\hat W) \|_\op
           + \| {M_3}(W,W,\hat W) - {M_3}(W, \hat W, \hat W)\|_\op \\
           &\quad 
           + \|{M_3}(W,\hat W,\hat W) - {M_3}(\hat W, \hat W, \hat W)\|_\op 
           + \|{M_3}(\hat W,\hat W,\hat W) - \hat {M_3}(\hat W,\hat W, \hat W)\|_\op \\
           &\le 
           \| {M_3}(W,W,W - \hat W) \|_\op
           + \| {M_3}(W,W - \hat W,\hat W) |_\op 
           + \|{M_3}(W - \hat W,\hat W,\hat W) |_\op \\
           &\le
           \| {M_3} \|_\op \|W\|^2_\op \aerr{W} +
            \| {M_3} \|_\op \|\hat W\|_\op \|W\|_\op \aerr{W} +
            \| {M_3} \|_\op \|\hat W\|^2_\op \aerr{W} +
            \aerr{M_3} \|\hat W\|^3_\op  \\
           &\le
           \| {M_3} \|_\op (\|W\|^2_\op + \|\hat W\|_\op \|W\|_\op + \|\hat W\|^2_\op) \aerr{W} +
            \aerr{M_3} \|\hat W\|^3_\op 
\end{align*}
We can relate $\|\hat W\|$ and $\aerr{W}$ to $\aerr{M_2}$ using 
\reflem{white}, for which we need the following condition,
\begin{condition}
  Let $\aerr{M_2} < \sigma_k(M_2)/3$.
\end{condition}

Then,
\begin{align*}
  \|\hat W\|_\op 
  &\le \frac{\sigma_k(M_2)^{-1/2}}{\sqrt{1 - \frac{\aerr{M_2}}{\sigma_k(M_2)} }} \\
                 &\le \sqrt{2} \sigma_k(M_2)^{-1/2} \\
  \aerr{W} 
  &\le 2 \sigma_k(M_2)^{-1/2} \frac{\frac{\aerr{M_2}}{\sigma_k(M_2)}}{1 - \frac{\aerr{M_2}}{\sigma_k(M_2)}} \\
           &\le 4 \sigma_k(M_2)^{-3/2} \aerr{M_2}.
\end{align*}
Thus,
\begin{align*}
  \aerr{T} &\le 
  6 \| {M_3} \|_\op \|W\|^2_\op (4 \sigma_k(M_2)^{-3/2}) \aerr{M_2} +
  \aerr{M_3} 2\sqrt{2} \|W\|^3_\op \\
  &\le 
  24 \| {M_3} \|_\op \sigma_k(M_2)^{-5/2} \aerr{M_2} +
  2\sqrt{2} \sigma_k(M_2)^{-3/2} \aerr{M_3} \\
  &\le 
    \sigma_k(M_2)^{-3/2}
      \left(24 \frac{\| {M_3} \|_\op}{\sigma_k(M_2)} + 2\sqrt{2} \right)
      \max\{ \aerr{M_2}, \aerr{M_3} \}.
\end{align*}

\paragraph{Step 2: Decomposition}

We have constructed $T$ to be a symmetric tensor with orthogonal
eigenvectors. We can now apply the results of \citet[Theorem
5.1]{AnandkumarGeHsu2012} to bound the error in the eigenvalues,
$\lambda_W$, and eigenvectors, $\omega$, returned by the robust tensor
power method;
\newcommand{\lW}{\lambda_W}
\newcommand{\lhW}{{\hat\lambda}_W}
\newcommand{\mW}{\omega}
\newcommand{\mhW}{{\hat\omega}}
\begin{align*}
  \|\lW - \lhW \|_{\infty} 
  &\le \frac{5 k \aerr{T}}{(\lW)_{\min}} \\
\|\mW_h -\mhW_h \|_2 
&\le \frac{8 k \aerr{T}}{(\lW)_{\min}^2},
\end{align*}
for all $h \in [k]$, where $(\lW)_{\min}$ is the smallest
eigenvalue of $T$. 

\paragraph{Step 3: Unwhitening}

Finally, we need to invert the whitening transformation to recover $\pi$ and
$\beta_h$ from $\lW$ and $\mW_h$. Let us complete the proof by
studying how this inversion relates the error in $\pi$ and $\beta$ to
the error in $\lW$ and $\mW$.

First, we will bound the error in the $\beta$s,
\begin{align*}
  \|\hat \beta_h - \beta_h\|_2
  &= \| \Whinv \mhW - \Winv \mW \|_2 \\
  &\le \aerr{\Winv} \|\mhW_h\|_2 + \|\Winv\|_2 \|\mhW_h - \mW_h \|_2. \comment{Triangle inequality}
\end{align*}

Once more, we can apply the results of \reflem{white}, 
\begin{align*}
  \|\Whinv\|_\op 
    &\le \sqrt{\sigma_1(M_2)} \sqrt{1 + \frac{\aerr{M_2}}{\sigma_k(M_2)} } \\
    \aerr{\Winv} 
    &\le 2 \sqrt{\sigma_1(M_2)} \left(\sqrt{1 + \frac{\aerr{M_2}}{\sigma_1(M_2)}}\right) \frac{\frac{\aerr{M_2}}{\sigma_k(M_2)}}{1 - \frac{\aerr{M_2}}{\sigma_k(M_2)}}.
\end{align*}
Using Condition 1, this simplifies to, 
\begin{align*}
  \|\Whinv\|_\op &\le \sqrt{3/2} \|M_2\|_\op^{1/2} \\
  \aerr{\Winv} &\le 4\sqrt{3/2} \|M_2\|_\op^{1/2} \sigma_k(M_2)^{-1} \aerr{M_2}.
\end{align*}

Thus,
\begin{align*}
  \|\hat \beta_h - \beta_h\|_2
  &\le 4\sqrt{3/2} \|M_2\|_\op^{1/2} \sigma_k(M_2)^{-1} \aerr{M_2} 
    + 8 \|M_2\|_\op^{1/2} \frac{k \aerr{T}}{(\lW)_{min}^2} \\
    &\le 4\sqrt{3/2} \|M_2\|_\op^{1/2} \sigma_k(M_2)^{-1} \aerr{M_2} \\
  &\quad + 8 \|M_2\|_\op^{1/2} k \pi_{\max} 
    \sigma_k(M_2)^{-3/2}
      \left(24 \frac{\| {M_3} \|_\op}{\sigma_k(M_2)} + 2\sqrt{2} \right)
      \max\{\aerr{M_2}, \aerr{M_3}\}.
\end{align*}

Next, let us bound the error in $\pi$,
\begin{align*}
  |\hat \pi_h - \pi_h |
  &= \left| \frac{1}{(\lW)_h^2} - \frac{1}{(\lhW)_h^2} \right| \\
  &= \left| \frac{\left( (\lW)_h + (\lhW)_h \right) \left( (\lW)_h - (\lhW)_h \right)}
  {(\lW)_h^2(\lhW)_h^2} \right| \\
  &\le \frac{( 2(\lW)_h - \|\lW - \lhW\|_{\infty} )}{(\lW)_h^2 \left( (\lW)_h + \|\lW - \lhW\|_{\infty} \right)^2} \|\lW - \lhW\|_{\infty}.
\end{align*}
To simplify the above expression, we would like that $\|\lW - \lhW
\|_{\infty} \le (\lW)_{\min}/2$ or $\|\lW - \lhW \|_{\infty} \le
\frac{1}{2\sqrt{\pi_{\max}}}$, recalling that $(\lW)_h = \pi_h^{-1/2}$. Thus,
we would like to require the following condition to hold on $\epsilon$;  
\begin{condition}
  $\epsilon \le \frac{1}{2\sqrt{\pi_{\max}}}$.
\end{condition}

Now,
\begin{align*}
  |\hat \pi_h - \pi_h |
  &\le \frac{(3/2)(\lW)_h}{(\lW)_h^4}
  \|\lW - \lhW\|_{\infty} \\
  &\le \frac{3}{2(\lW)_h^3} \frac{5 k \aerr{T}}{(\lW)_{\min}^2} \\
  &\le \frac{3 \pi_{\max}^{3/2}}{2} 5 k \pi_{\max} 
    \sigma_k(M_2)^{-3/2}
    \left(24 \frac{\| {M_3} \|_\op}{\sigma_k(M_2)} + 2\sqrt{2} \right) \max\{\aerr{M_2}, \aerr{M_3}\} \\
    &\le \frac{15}{2} \pi_{\max}^{5/2} k 
    \sigma_k(M_2)^{-3/2}
    \left(24 \frac{\| {M_3} \|_\op}{\sigma_k(M_2)} + 2\sqrt{2} \right) \max\{\aerr{M_2}, \aerr{M_3}\}.
\end{align*}

Finally, we complete the proof by requiring that the bounds $\aerr{M_2}$ and
$\aerr{M_3}$ imply that $\|\hat \pi - \pi \|_{\infty} \le \epsilon$ and
$\|\hat \beta_h - \beta_h\|_2 \le \epsilon$, i.e.
\begin{align*}
%  \frac{15}{2} \pi_{\max}^{5/2} k 
%  \sigma_k(M_2)^{-3/2}
%  \left(24 \frac{\| {M_3} \|_\op}{\sigma_k(M_2)} + 2\sqrt{2} \right) \max\{\aerr{M_2}, \aerr{M_3}\} \le \epsilon \\
  \max\{\aerr{M_2}, \aerr{M_3}\} &\le
  \left( \frac{15}{2} \pi_{\max}^{5/2} k 
  \sigma_k(M_2)^{-3/2}
  \left(24 \frac{\| {M_3} \|_\op}{\sigma_k(M_2)} + 2\sqrt{2} \right) \right)^{-1} \epsilon \\
%  \frac{%
%    2 \sigma_k(M_2)^{3/2}
%   }{%
%    15 k \pi_{\max}^{5/2} \left(24 \frac{\| {M_3} \|_\op}{\sigma_k(M_2)} + 2\sqrt{2} \right)
%   }~ \epsilon \\
% \end{align*}
% and
% \begin{align*}
%   4\sqrt{3/2} \|M_2\|_\op^{1/2} \sigma_k(M_2)^{-1} \aerr{M_2}
%     + 8 \|M_2\|_\op^{1/2} k \pi_{\max} 
%       \sigma_k(M_2)^{-3/2}
%         \left(24 \frac{\| {M_3} \|_\op}{\sigma_k(M_2)} + 2\sqrt{2} \right)
%       \max\{\aerr{M_2}, \aerr{M_3}\} \le \epsilon \\
  \max\{\aerr{M_2}, \aerr{M_3}\} &\le
  \left( 4\sqrt{3/2} \|M_2\|_\op^{1/2} \sigma_k(M_2)^{-1} \aerr{M_2}
    + 8 \|M_2\|_\op^{1/2} k \pi_{\max} 
      \sigma_k(M_2)^{-3/2}
        \left(24 \frac{\| {M_3} \|_\op}{\sigma_k(M_2)} + 2\sqrt{2} \right)
        \right)^{-1} \epsilon.
\end{align*}

\end{proof}

\section{Basic Lemmas}

In this section, we have included some standard results that we employ for completeness.

\begin{lemma}[Concentration of vector norms]
  \label{lem:conc-norms}
  Let $X, X_1, \cdots, X_n \in \Re^d$ be i.i.d.\ samples
  from some distribution with bounded support
  ($\|X\|_2 \le M$ with probability 1).
  Then with probability at least $1 - \delta$,
  \begin{align*}
    \left\| \frac{1}{n} \sum_{i=1}^{n} X_i - \E[X] \right\|_2 &
    \le \frac{2M}{\sqrt{n}} \left(1 + \sqrt{\frac{\log(1/\delta)}{2}}\right).
  \end{align*}
\end{lemma}
\begin{proof}
  Define $Z_i = X_i - \E[X]$.

%  First, $\|Z_i\|_2 \le \|X_i\|_2 + \|\E[X_i]\|_2 \le 2M$,
%  where the first inequality is due to the triangle inequality,
%  and the second follows by Jensen's inequality on $\|\cdot\|_2$
%  and the boundedness assumption on $X_i$.

The quantity we want to bound can be expressed as follows:
  \begin{align*}
  f(Z_1, Z_2, \cdots, Z_n) = \left\| \frac1n \sum_{i=1}^n Z_i \right\|_2.
  \end{align*}

Let us check that $f$ satisfies the bounded differences inequality:
  \begin{align*}
|f(Z_1, \cdots, Z_i, \cdots, Z_n) - f(Z_1, \cdots, Z_i', \cdots, Z_n)|
& \le \frac1n \|Z_i - Z_i'\|_2 \\
& = \frac1n \|X_i - X_i'\|_2 \\
&\le \frac{2M}{n},
  \end{align*}
  by the bounded assumption of $X_i$ and the triangle inequality.

By McDiarmid's inequality,
with probability at least $1 - \delta$,
we have:
\begin{align*}
\Pr[f - \E[f] \ge \epsilon] \le
\exp\left(\frac{-2 \epsilon^2}{\sum_{i=1}^n (2M/n)^2}\right).
\end{align*}
Re-arranging:
\begin{align*}
  \left\|\frac{1}{n}\sum_{i=1}^n Z_i\right\|_2
  &\le \E\left[ \left\| \frac{1}{n} \sum_{i=1}^n Z_i \right\|_2 \right]
  + M\sqrt{\frac{2\log(1/\delta)}{n}}.
\end{align*}

Now it remains to bound $\E[f]$.
By Jensen's inequality, $\E[f] \le \sqrt{\E[f^2]}$,
so it suffices to bound $\E[f^2]$:
\begin{align*}
  \E\left[ \frac1{n^2} \left\| \sum_{i=1}^n Z_i \right\|^2 \right]
  &= \E\left[ \frac1{n^2} \sum_{i=1}^n \|Z_i\|_2^2 \right] +
\E\left[ \frac1{n^2} \sum_{i\neq j} \innerp{Z_i}{Z_j} \right] \\ 
& \le \frac{4M^2}{n} + 0,
\end{align*}
where the cross terms are zero by independence of the $Z_i$'s.

Putting everything together, we obtain the desired bound:
\begin{align*}
\left\|\frac{1}{n}\sum_{i=1}^n Z_i \right\|
&\le \frac{2M}{\sqrt{n}} + M \sqrt{\frac{2\log(1/\delta)}{n}}.
\end{align*}
\end{proof}

%%%%%%%%%%%%%%%%%%%%%%%%%%%%%%%%%%%%%%%%%%%%%%%%%%%%%%%%%%%%

\textbf{Remark}: The above result can be directly
applied to the Frobenius norm of a matrix
$M$ because $\|M\|_F = \|\vvec(M)\|_2$.

\begin{lemma}[Perturbation Bounds on Whitening Matrices]
  \label{lem:white}
  Let $A$ be a rank-k $d\times d$ matrix, $\Wp$ be a $d \times k$ matrix that
  whitens $\hat A$, i.e. $\Wp^T \Ap \Wp = I$.  Suppose $\Wp^T A \Wp
  = U D U^T$, then define $W = \hat{W} U D^{-\half} U^T$. Note that $W$
  is also a $d \times k$ matrix that whitens $A$. If $\serr{A}
  = \frac{\aerr{A}}{\sigma_k(A)} < \frac{1}{3}$ 
  then, 
  \begin{align*}
    \|\hat W\|_\op 
      &\le \frac{\|W\|_\op}{\sqrt{1 - \serr{A}} } \\
  \|\Whinv\|_\op 
    &\le \|\Winv\|_\op \sqrt{1 + \serr{A}} \\
    \aerr{W} 
      &\le \|W\|_\op \frac{\serr{A}}{1 - \serr{A}} \\
    \aerr{\Winv} 
      &\le \|\Winv\|_\op \sqrt{1 + \serr{A}} \frac{\serr{A}}{1 - \serr{A}}.
  \end{align*}
\end{lemma}
\begin{proof}
  This lemma has also been proved in \citet[Lemma 10]{hsu13spherical},
  but we present it differently here for completeness.
  First, note that for a matrix $W$ that whitens $A = V \Sigma V^T$,
  $W = V \Sigma^{-\half} V^T$ and $\Winv = V \Sigma^{\half} V^T$. Thus, by rotational invariance,
  \begin{align*}
    \|W\|_\op &= \frac{1}{\sqrt{\sigma_k(A)}} \\
    \|\Winv\|_\op &= \sqrt{\sigma_1(A)} \\
  \end{align*}
  This allows us to bound the operator norms of $\hat W$ and $\Whinv$ in
  terms of $W$ and $\Winv$,
  \begin{align*}
    \|\hat W\|_\op &= \frac{1}{\sqrt{\sigma_k(\hat A)}} \\
    &\le \frac{1}{\sqrt{\sigma_k({A}) - \aerr{A}} } \comment{By Weyl's Theorem} \\
    &\le \frac{1}{1 - \serr{A}} \frac{1}{\sqrt{\sigma_k(A)}} \\
    &= \frac{\|W\|_\op}{\sqrt{1 - \serr{A}} } \\
    \|\Whinv\|_\op &= \sqrt{\sigma_1(\hat A)} \\
    &\le \sqrt{\sigma_1({A}) + \aerr{A}}  \comment{By Weyl's Theorem} \\
    &\le \sqrt{1 + \serr{A}} \sqrt{\sigma_1(A)} \\
    &= \sqrt{1 + \serr{A}} \|\Winv\|_\op.
  \end{align*}

  To find $\aerr{W}$, we will exploit the rotational invariance of the operator norm. 
  \begin{align*}
    \aerr{W} &= \| \Wp - W \|_\op \\
    &= \| W U D^{\half} U^T - W \|_\op  \comment{$\hat W = W U D^{-\half} U^T$}\\
    &\le \|W\|_\op \| I - U D^{\half} U^T \|_\op \comment{Sub-multiplicativity}.
  \end{align*}

  We will now bound $\| I - U D^{\half} U^T \|_\op$. Note that by
  rotational invariance, $\| I - U D^{\half} U^T \|_\op = \|
  I - D^{\half} \|_\op$. By Weyl's inequality, $|1 - \sqrt{D_{ii}}| \le
  \| I - D^{\half} \|_\op$; put differently, $\| I - D^{\half} \|_\op$
  bounds the amount $\sqrt{D_{ii}}$ can diverge from $1$. Using the
  property that $|(1+x)^{-1/2} - 1| \le |x|$ for all $|x| \le 1/2$, we
  will take an alternate approach and bound $D_{ii}$ separately, and use
  it to bound $\| I - D^{\half} \|_\op$.
  \begin{align*}
    \| I - D \|_\op 
    &= \| I - U D U^T \|_\op \comment{Rotational invariance} \\
    &= \| \hat W^T \hat A \hat W - \hat W^T A \hat W \|_\op \comment{By definition} \\
    &= \| \hat W^T (\hat A - A) \hat W \|_\op \\
    &\le \|\hat W\|_\op^2 \aerr{A} \\
    &\le \frac{1}{\sigma_k(A)} \frac{\aerr{A}}{1 - \serr{A}} \\
    &\le \frac{\serr{A}}{1 - \serr{A}}.
  \end{align*}

  Therefore, if $\frac{\serr{A}}{1 - \aerr{A}} < 1/2$, or $\aerr{A} < \frac{\sigma_k(A)}{3}$,
  \begin{align*}
    \| I - U D^{1/2} U^T \|_\op &\le \frac{\serr{A}}{1 - \serr{A}}.
  \end{align*}

  We can now complete the proof of the bound on $\aerr{W}$, 
  \begin{align*}
    \aerr{W} 
    &= \|W\|_\op \| I - U D^{1/2} U^T \|_\op \comment{Rotational invariance}\\
    &\le \|W\|_\op \frac{\serr{A}}{1 - \serr{A}}.
  \end{align*}

  Similarly, we can bound the error on the un-whitening transform, $\Winv$,
  \begin{align*}
    \aerr{\pinv{W}} &= \| \pinv{\Wp} - \pinv{W} \|_\op \\
    &= \| \pinv{\Wp} - U D^{\half} U^T \pinv{\Wp} \|_\op \comment{$\pinv{W} = U D^{1/2} U^T \pinv{\Wp}$} \\
    &\le \|\pinv{\Wp}\|_\op \| I - U D^{\half} U^T \|_\op \\
    &\le \|\pinv{W}\|_\op \sqrt{1+\serr{A}} \frac{\serr{A}}{1 - \serr{A}}.
  \end{align*}
\end{proof}

}{}

\end{document}